\documentclass{l4dc2026}

\usepackage{amsmath, amssymb,amsfonts,color,nicefrac}
\usepackage{xcolor}
\usepackage{bm}
\usepackage{bbm}
\usepackage{algpseudocode}
\usepackage{comment}
\usepackage{caption}
\usepackage{float}
\usepackage[table]{xcolor}
\usepackage{pgf}

\newcommand{\heatcell}[3]{%
  \pgfmathsetmacro{\heatpct}{%
    max(0,min(100,100*(#1-#2)/(#3-#2)))%
  }%
  \edef\heatcolor{%
    \noexpand\cellcolor{black!\heatpct!white!30}%
  }%
  \heatcolor #1%
}

\title[Collaborative Navigation and Exploration with $\beta$-SGP]{Collaborative Navigation and Exploration with \\ $\beta$-Sparse Gaussian Processes}
\usepackage{times}

\author{%
 \Name{Evangelos Psomiadis} \Email{epsomiadis3@gatech.edu}\\
 \addr {D. Guggenheim School of Aerospace Engineering, Georgia Tech, Atlanta, GA, USA}
 \AND
 \Name{Dipankar Maity} \Email{dmaity@charlotte.edu}\\
 \addr {Department of Electrical and Computer Engineering, UNC Charlotte, Charlotte, NC, USA}
 \AND
 \Name{Panagiotis Tsiotras} \Email{tsiotras@gatech.edu}\\
 \addr {D. Guggenheim School of Aerospace Engineering, Georgia Tech, Atlanta, GA, USA}
}

\DeclareMathOperator{\E}{\mathbb{E}}

\DeclareMathOperator*{\argmin}{arg\,min}
\renewcommand{\Re}{\mathbb{R}}

\begin{document}
\typeout{TEXTWIDTH = \the\textwidth}
\typeout{PAPERWIDTH = \the\paperwidth}
\typeout{ODDSIDEMARGIN = \the\oddsidemargin}

\maketitle



\begin{abstract} 
Collaborative navigation of heterogeneous robots in unknown environments poses significant challenges due to sensing, communication, and computational limitations. 
In this work, a lead robot navigates toward a target with the assistance of a mobile sensor robot (e.g., an airborne drone), whose task is to explore the environment and communicate relevant map information under bandwidth constraints.
We propose a framework in which the sensor determines both which map regions to transmit and which navigation action to execute online, while also predicting unexplored regions of the environment. 
To this end, we present \emph{$\beta$-Sparse Gaussian Processes}, a variational sparse Gaussian Process (GP) model for task-aware inducing point selection under cardinality constraints. 
Furthermore, we develop an action-selection strategy that balances task relevance with exploration.
Simulations on Mars and Earth maps show that the proposed framework outperforms all considered bandwidth-constrained baselines, reducing accumulated path cost by approximately $35\%$ compared to no communication, while decreasing transmitted information by $94\%$ relative to the full-information GP baseline.
The results further highlight the benefits of GP-based map prediction over methods without prediction, as well as the advantages of navigation-aware region selection over approaches that do not prioritize task-relevant regions.
\end{abstract}

\begin{keywords}
Multi-Agents, Sparse Gaussian Processes
\end{keywords}


\section{Introduction}
Robotic agents are required to navigate and explore large, unknown environments while communicating efficiently to achieve team-level objectives.
Across such scenarios, three fundamental questions arise: \textit{how is the environment modeled and perceived}, \textit{what information is shared among the agents}, and \textit{how are the agents’ actions determined}?
In this work, we employ Gaussian Processes (GPs) to model the environment and address the latter two questions using this probabilistic model.

Gaussian Processes are widely used for modeling functions due to their simplicity and ability to explicitly quantify uncertainty~\citep{Rasmussen2006}.
For example, GPs have been used to model wildfires~\citep{Du2025}, as well as occupancy maps, improving navigation performance by predicting unexplored regions~\citep{GhaffariJadidi2018}.
When multiple robots share a common objective, their measurements can be combined into a unified GP. 
However, under communication constraints, transmitted measurements should be carefully selected based on the team objective.

Selecting a subset of data for GP representations has been extensively studied in the machine learning community. 
Although these studies primarily address GP scalability~\citep{wenger2024}, whereas our problem stems from communication constraints in robotic systems, the two problems admit similar solutions.
Early approaches employed low-rank kernel approximations to represent GPs using fewer points~\citep{Williams2001Nystrom}.
\citet{Titsias09} introduced a variational formulation for Sparse Gaussian Process Regression (SGPR), which \citet{Hensman2013BigData} subsequently extended to Stochastic Variational Gaussian Processes (SVGP).
In~\citep{Wei2021DLM}, the authors proposed a direct loss minimization framework for sparse GPs, combining predictive loss with KL regularization.

While these methods effectively produce compressed representations of GPs, they do not explicitly account for the task-dependent relevance of the retained information. 
The work of~\citep{Hensman2015MCMC} enabled more informative approximations by introducing a prior over the inducing inputs, which can be tailored to task-specific properties.
Building on this, subsequent approaches explored alternative priors~\citep{rossi2021sparsegp, Uhrenholt2021Probabilistic, meng2021rsgp}.
In contrast, we employ a cardinality-conditioned Bernoulli point process over inducing inputs to prioritize regions of interest through regularization.

We focus on multi-robot operations in bandwidth-constrained environments, where a mobile sensor assists a navigating robot by transmitting map information.
To select informative observations, we introduce $\beta$-Sparse Gaussian Processes ($\beta$-SGPs), a sparse GP approximation inspired by the Information Bottleneck~\citep{Tishby1999IB} and $\beta$-Variational Autoencoders~\citep{higgins2017betavae}.
We expect this model to be of interest not only to the robotics community but also to the machine learning community at large.
In addition, we develop an informative exploration strategy for the sensor robot inspired by the GP-UCB algorithm~\citep{Srinivas2012, Bogunovic2016TimeVaryingGP}.

The main contributions of this paper are summarized as follows: First, we introduce $\beta$-SGP, a Sparse Gaussian Process model for region-aware GP compression under cardinality constraints; second, we develop a framework for collaborative online navigation and exploration using $\beta$-SGP for selecting transmitted points under bandwidth constraints along with a strategy for action selection; third, we conduct extensive simulations on real-world maps, showing that the proposed framework improves navigation performance over existing methods while satisfying bandwidth constraints\footnote{We will make the code publicly available upon acceptance.}.

\begin{figure}[tb]
    \centering        
    {\includegraphics[width=0.85\linewidth]{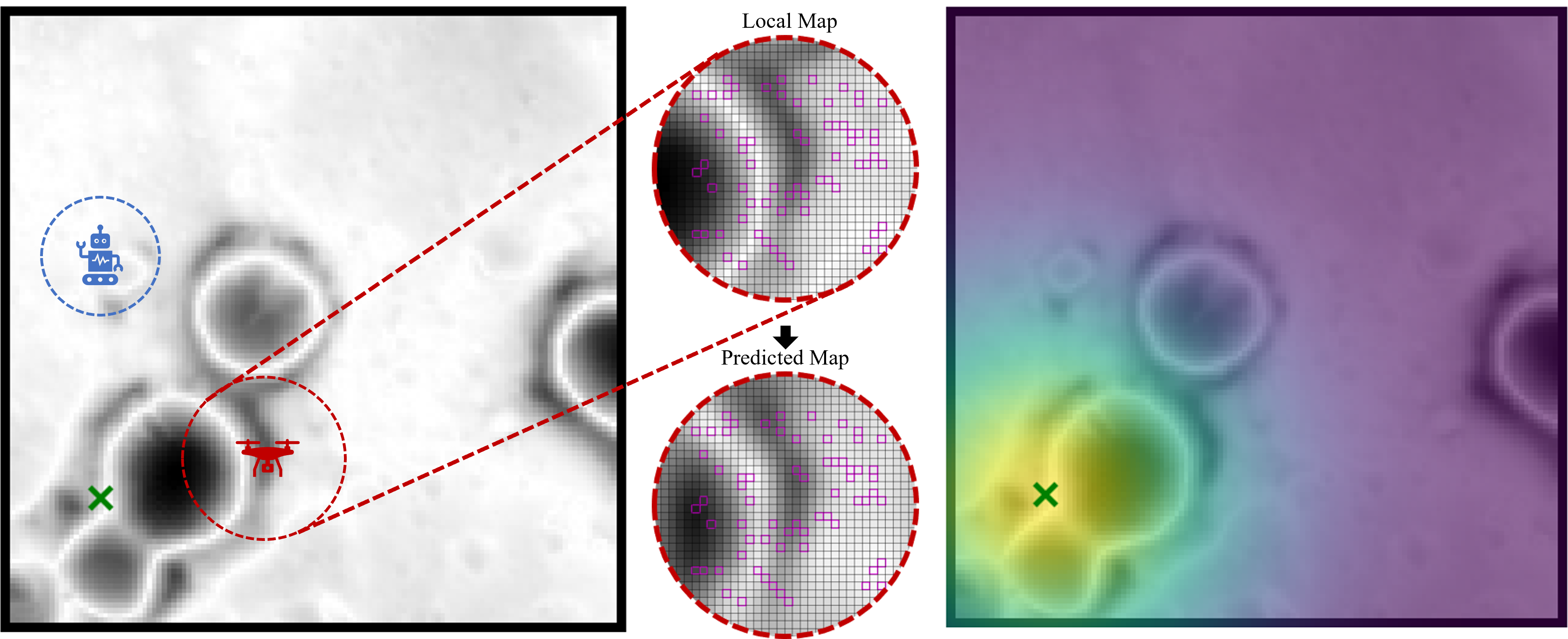}}
    \caption{[left] Mars slope map~\citep{HiRISE2024} with two robots: the Actor (blue) and the Sensor (red) along with the Actor's goal $\textcolor{green!50!black}{\pmb{\pmb{\times}}}$. [middle] The Sensor selects 70 representative points from its local map (magenta squares in ``Local Map") and transmits them to the Actor, which reconstructs a predicted map using Gaussian Processes. [right] Path-agnostic Region of Interest.}
    \label{fig:mars_env}
\end{figure}

\section{Related Work}
Several researchers have studied collaborative navigation problems involving heterogeneous robot teams.
In~\citep{Sasaki2020WhereToMap}, a rover-copter team is considered, where the copter reduces the rover’s odometry uncertainty.
In~\citep{Psomiadis2024DARS, psomiadis2023, Psomiadis2025}, mobile sensors perform informative communication and exploration to support ground-vehicle navigation, while~\citep{pedram2025, suthar2025} address similar problems using rate-distortion theory and Mixed-Integer Linear Programming, respectively.
Similarly,~\citep{Larsson2025} proposes a framework for time-sequential compression of dynamic probabilistic grid maps.
\cite{Damigos2024} introduce a communication framework for 3D point clouds.
However, these works do not leverage GPs to predict unexplored map regions and improve navigation performance.

Along this direction, several works have incorporated GPs into related problems.
In~\citep{Kai-cheih2018}, sparse GPs are used for informative path planning in environmental monitoring, while~\citep{Jakkala2024MultiRobotIPP} extend this setting to multiple robots under constraints.
In~\citep{STLAURENT2023110966}, GPs are used for joint path planning and sensor configuration, with offline path planning.
More recently,~\citep{Du2025} proposed an active sensing and coverage framework for spatial fields using online random feature GPs.
However, existing works do not consider task-dependent compression of the GP under cardinality constraints as is done in this work.

\section{Preliminaries and Notation}
 \label{sec:preliminaries}

We model the environment as a static 2D grid $\mathcal X \subset \mathbb R^2$ of $N$ cells (e.g., an occupancy grid~\citep{elfes1987,moravec1988}, elevation map, or cost map).
Figure~\ref{fig:mars_env} shows an example Mars slope map~\citep{HiRISE2024}.
Each cell $\mathbf{x}_i \in \mathcal X$ has a value $f(\mathbf{x}_i)\in[0,1]$, where, for a slope map, $0$ and $1$ represent the minimum and maximum slopes, respectively.
Robot $r$ measures $y_i^r=f(\mathbf{x}_i)+\varepsilon_i^r$, with $\varepsilon_i^r\sim\mathcal N(0,(\sigma^r)^2)$, and selects actions $v_t^r\in U$ from a finite control set $U$.
When it is clear, we omit the robot superscript.
We use $p(\cdot)$ for model distributions and $q(\cdot)$ for variational distributions.

\section{Problem Formulation and Framework Overview}

Consider a pair of mobile robots (an Actor and a Sensor) that move through an unknown environment \( \mathcal X \subset \Re^2\), as described in Section~\ref{sec:preliminaries}.
Both agents operate concurrently and, at each timestep \( t \), perceive a noisy local map of the environment. 
The Actor's goal is to reach a known target while minimizing a cost by using an online path-planning algorithm. 
The Sensor aims to assist the Actor by sending its observations through a network with varying bandwidth.
We assume that the Sensor is an aerial vehicle, and thus its motion is independent of the cell values, whereas the Actor must account for them (e.g., ground vehicle).
We denote the Actor by $r = A$ and the Sensor by $r = S$.

\subsection{Problem Statement}

The goal is to develop a joint communication and navigation framework in which the Sensor assists the Actor to reach its target with minimum cost. 
The framework must satisfy a time-varying constraint on the amount of information transmitted by the Sensor, reflecting the available communication bandwidth. 
The Actor can also communicate limited information to the Sensor.
The inputs are the initial Actor and Sensor positions, \( \mathbf{x}_0^{A} \) and \( \mathbf{x}_0^{S} \), the Actor's target \( \mathbf{x}_G^{A} \), and the perception noise \( \varepsilon^{r} \).
Additionally, at each timestep $t$, each robot $r$ collects  a set of map measurements denoted by 
\( (\mathbf{X}, \mathbf{y}(\mathbf{X}))_{t}^{r} \), where $\mathbf{X}^r = [\mathbf{x}_1^\top, \dots, {\mathbf{x}_{n^{r}}}^\top]^\top$ are the map locations observed by robot \(r \in \{A,S\}\).

\subsection{Framework Overview}

In our proposed framework, we use GPs to improve navigation performance by predicting unexplored regions. 
We model the environment as a latent function $f$ sampled from a Gaussian Process (see Appendix~\ref{app:GP_all}),
i.e., \( f(\mathbf{x}) \sim \mathcal{GP}\big(\mu(\mathbf{x}), k_{\boldsymbol{\theta}} (\mathbf{x}, \mathbf{x}')\big) \),  where \( \mathbf{x}, \mathbf{x}' \in \mathcal{X} \), and the kernel $k$ is parameterized by hyperparameters \( \boldsymbol{\theta} \).
The Sensor selects which observations to transmit using the proposed \(\beta\)-Sparse GP model and determines its next motion action through an algorithm that balances exploration and exploitation. 
The Actor computes its control action \( v_t^{A} \) using both its own measurements and the information received from the Sensor.
An overview of the proposed framework is shown in Figure~\ref{fig:framework}. 
In the following sections, we describe its main components. 
Additional details are provided in Appendix~\ref{app:components}, while the Sensor's algorithm is summarized in Appendix~\ref{app:algorithm}.

\begin{figure}[tb]
    \centering        
    {\includegraphics[width=1.0\linewidth]{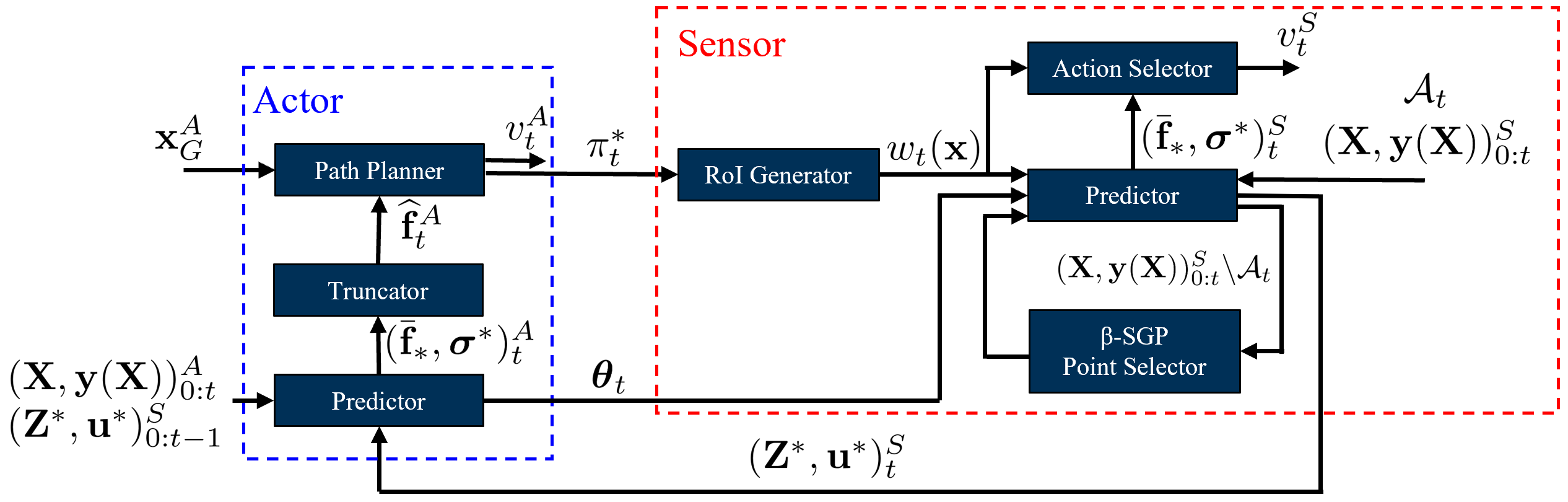}}
    \caption{Proposed framework schematic showing the Actor and Sensor components at timestep $t$.} 
    \label{fig:framework}
\end{figure}

\section{Region of Interest (RoI)}\label{sec:RoI}

The primary objective of the mission is for the Actor to reach its target with minimum accumulated cost. 
To effectively support this objective, the Sensor must account for the Actor's task when selecting both the data points to transmit and its own actions. 
To this end, we define a Region of Interest (RoI) through a weight function $w_t(\mathbf{x})$ that captures map regions that are relevant to the Actor's task. 
We consider two formulations: (i) a path-agnostic formulation, which reduces communication requirements at the expense of navigation performance, and (ii) a path-dependent formulation, which improves navigation performance at the cost of increased Actor communication.

\textbf{Path-agnostic:} We assign higher weights to points $\mathbf{x}$ near the Actor's goal $\mathbf{x}_G^A$ using an unnormalized Gaussian with spatial spread controlled by $\tilde{\sigma}^2$ (see Figure~\ref{fig:mars_env}), that is,
\begin{equation}\label{eq:ROI_weights_simple}
    w_t(\mathbf{x}) = w(\mathbf{x}) = e^{-\frac{\| \mathbf{x} -\mathbf{x}_G^{A}\|^2}{2 \tilde{\sigma}^2}}.
\end{equation}

\textbf{Path-dependent:} The Actor transmits its planned path $\pi_t^*$ (or representative waypoints) to the Sensor, which constructs the RoI prioritizing regions near the path~\citep{psomiadis2023}, that is,
\begin{equation}\label{eq:ROI_weights_path}
    w_t(\mathbf{x}) = \max_{\textbf{x}^* \in \pi^*_t} e^{-\frac{\| \mathbf{x} -\mathbf{x}^{*}\|^2}{2 \tilde{\sigma}^2}}.
\end{equation}

\section{Selection of Transmitted Sensor Points}

The RoI weights, together with the communication limitations and the GP model, are leveraged by the Sensor to select, at each timestep \( t \), the most informative subset of observations to transmit.
Recall that each Sensor observation \( y_i^S(\mathbf{x}_i) \) is corrupted by Gaussian noise $\varepsilon_i^S \sim \mathcal{N}(0, (\sigma^S)^2).$

Let $(\mathbf{X}, \mathbf{y}(\mathbf{X}))_{0:t}^{S}$ denote the set of observations collected by the Sensor up to timestep $t$.
Although the following formulation allows the Sensor to use these observations to obtain its own GP hyperparameters, the Actor instead transmits its GP hyperparameters $\boldsymbol{\theta}$ to the Sensor (see Figure~\ref{fig:framework}) to reduce computational overhead and ensure consistency between their environment representations.
For clarity of exposition, we omit time and robot superscripts in the remainder of this section.

Appendix~\ref{app:GP} provides an overview of Gaussian Processes, including inference and hyperparameter optimization. 
Briefly, training a GP involves maximizing the log marginal likelihood (LML) $\log p(\mathbf{y}  |  \mathbf{X}, \boldsymbol{\theta}),$
which incurs a computational cost of \( \mathcal{O}(n^3) \), where \(n\) is the number of training observations. 
To address this limitation, Sparse GPs (SGPs) have been developed~\citep{Titsias09, Hensman2013BigData} (see Appendix~\ref{app:SGPR} and ~\ref{app:SVGP}), which approximate the full model using a reduced set of $m \ll n$ inducing points, significantly improving computational efficiency via maximization of the Evidence Lower Bound (ELBO) to LML.
Although SGPs are primarily designed to reduce computational complexity, we leverage their ability to represent a GP using a small set of informative points. 
Specifically, the Sensor adopts an SGP formulation to select which measurements to transmit to the Actor under communication constraints, rather than to accelerate GP inference.

\subsection[Beta-Sparse Gaussian Processes]{$\beta$-Sparse Gaussian Processes}

Let $\mathcal{S}_t=\{\mathbf{z}_1,\ldots,\mathbf{z}_M\}$ denote the set of candidate inducing locations, and let $\mathbf{Z}\subseteq\mathcal{S}_t$ denote the selected inducing-point subset, with corresponding function values $\mathbf{u}=f(\mathbf{Z})$.
We adopt a variational formulation to approximate the GP, where the inducing variables are treated as variational parameters. 
Our goal is to find the inducing points that both accurately represent the underlying GP and preserve information relevant to the RoI. 
For this, we use a cardinality-conditioned multivariate Bernoulli distribution $p(\mathbf{Z})$ (see Appendix~\ref{app:conditional_bernoulli}) as a prior over inducing locations:
\begin{align} \label{eq:p_z}
p(\mathbf Z)=
\begin{cases}
\displaystyle\frac{1}{A_p}
\prod_{\mathbf z_i\in\mathbf Z} \frac{p_i}{1-p_i},
&|\mathbf Z|=m,\\[6pt]
0,&\text{otherwise},
\end{cases}
\end{align}
where $p_i/(1-p_i) = w(\mathbf z_i)$, 
$A_p=
\sum_{\substack{\mathbf S \subseteq  \mathcal S_t\\|\mathbf S|=m}}
\prod_{\mathbf z_i\in\mathbf S}w(\mathbf z_i)$, and $m = m_t$ is determined by the bandwidth, assuming the communication cost is proportional to the number of points sent.

We approximate the true posterior $p(\mathbf{f},\mathbf{u},\mathbf{Z} | \mathbf{X},\mathbf{y})$ using the variational distribution 
$q(\mathbf{f},\mathbf{u},\mathbf{Z})=p(\mathbf{f} | \mathbf{X},\mathbf{u},\mathbf{Z})q(\mathbf{u} | \mathbf{Z})q(\mathbf{Z})$.
For simplicity, we omit conditioning on $\boldsymbol{\theta}$.
To avoid a prohibitive combinatorial search over subsets of $\mathcal S_t$, we associate each point $\mathbf{z}_i \in \mathbf{Z}$ with a Bernoulli parameter $\lambda_i \in (0,1)$.
Similarly to $p(\mathbf{Z})$, we define $q(\mathbf{Z}) = q_{\lambda}(\mathbf{Z})$ as 
\begin{align} \label{eq:q_Z}
q_{\boldsymbol\lambda}(\mathbf Z)=
\begin{cases}
\displaystyle\frac{1}{A_q}
\prod_{\mathbf z_i\in\mathbf Z}
\frac{\lambda_i}{1-\lambda_i},
&|\mathbf Z|=m,\\[6pt]
0,&\text{otherwise},
\end{cases}
\end{align}
where $A_q=
\sum_{\substack{\mathbf S\subseteq \mathcal S_t \\|\mathbf S|=m}}
\prod_{\mathbf z_i\in\mathbf S}
\lambda_i/(1-\lambda_i)$.

To encourage the inducing points to remain relevant to the prior distribution, we impose a constraint on the divergence between their variational distribution and the prior
\begin{align}\label{eq:KL_constraint}
\mathrm{KL}\big(q(\mathbf{Z}) \,\|\, p(\mathbf{Z})\big) \le \alpha,
\end{align}
where $\alpha > 0$ controls the allowable deviation.

Using the distributions in~\eqref{eq:p_z} and~\eqref{eq:q_Z}, which inherently satisfy the bandwidth constraint, we seek a lower bound on the log marginal likelihood that depends on the variational distribution $q(\mathbf{f},\mathbf{u},\mathbf{Z})$ while incorporating the KL constraint in~\eqref{eq:KL_constraint} as a soft constraint to capture task relevance.
The following theorem establishes this ELBO, which we optimize to obtain the inducing points.

\begin{theorem}\label{thm:betaSGP_ELBO}
Consider the variational distribution $q(\mathbf{f}, \mathbf{u}, \mathbf{Z}) = p(\mathbf{f} | \mathbf{X},  \mathbf{u}, \mathbf{Z}) \, q(\mathbf{u} | \mathbf{Z}) \, q(\mathbf{Z})$.
Incorporating the KL constraint in~\eqref{eq:KL_constraint} through a Lagrangian relaxation, the following objective is an ELBO to the log marginal likelihood for any $\beta\geq 1$:
\begin{align}\label{eq:betaSGP-elbo}
    \mathcal{F}_2
    = \mathbb{E}_{q(\mathbf{Z})}\big[\mathcal{F}_1\big]
    - \beta \, \mathrm{KL}\big(q(\mathbf{Z}) \,\|\, p(\mathbf{Z})\big),
\end{align}
where,
\begin{align}\label{eq:SGP-elbo}
    \mathcal{F}_1
    = \mathbb{E}_{q(\mathbf{f}, \mathbf{u}  |  \mathbf{Z})}
    \big[\log p(\mathbf{y}  |  \mathbf{f})\big]
    - \mathrm{KL}\big(q(\mathbf{u}  |  \mathbf{Z})
    \,\|\, p(\mathbf{u}  |  \mathbf{Z})\big),
\end{align}
and
$q(\mathbf{f}, \mathbf{u}  |  \mathbf{Z})
= p(\mathbf{f}  |  \mathbf{X}, \mathbf{u}, \mathbf{Z})
q(\mathbf{u}  |  \mathbf{Z})$.
The parameter $\beta$ controls the strength of regularization toward the RoI prior $p(\mathbf{Z})$.
\end{theorem}
\begin{proof}
See Appendix~\ref{app:ThmbetaSGP_ELBO}.
\end{proof}

\subsection{ELBO Optimization and Transmitted Sensor Points}

Under the cardinality-conditioned Bernoulli distributions defined in \eqref{eq:p_z} and \eqref{eq:q_Z}, the KL divergence in~\eqref{eq:betaSGP-elbo} admits an analytical expression.

\begin{lemma}\label{lemma:KL_ELBO}
The KL divergence between $p(\mathbf{Z})$ and $q_{\boldsymbol{\lambda}}(\mathbf{Z})$ from~\eqref{eq:p_z} and~\eqref{eq:q_Z}, respectively, is given by
    \begin{align}\label{eq:KL_closed_form}
        \mathrm{KL}\big(
        q_{\boldsymbol{\lambda}}(\mathbf{Z})
        \,\|\,p(\mathbf{Z})\big) =
        \sum_{i=1}^{M}\bar q_i
        \left[
        \log\frac{\lambda_i}{1-\lambda_i}
        -\log w(\mathbf{z}_i)
        \right]
        +\log\frac{A_p}{A_q},
    \end{align}
    where $M=|\mathcal S_t|$ and
    $\bar q_i = \mathbb{P}_{q_{\boldsymbol{\lambda}}} (\mathbf{z}_i\in\mathbf{Z})$
    is the marginal inclusion probability.
\end{lemma}
\begin{proof}
    See Appendix~\ref{app:KL_ELBO}.
\end{proof}

The terms $\bar q_i$, $A_q$, and $A_p$ in~\eqref{eq:KL_closed_form} can be computed efficiently in $O(Mm)$ time using dynamic programming~\citep{Chen1997ConditionalBernoulli}.
Specifically, define $D(k,j,\mathbf{d}) = \sum_{\substack{S\subseteq\{1,\ldots,k\}\\|S|=j}} \prod_{i\in S}d_i,$ where $\mathbf{d}\in\mathbb{R}_{>0}^M$ is a vector of weights and consider the recursion
\begin{align}\label{eq:DP_recursion}
    D(k,j,\mathbf{d})
    &=D(k-1,j,\mathbf{d})
    +d_kD(k-1,j-1,\mathbf{d}),
\end{align}
with $D(0,0,\mathbf{d})=1$ and $D(k,j,\mathbf{d})=0$ for invalid indices.

The normalization constants are then $A_q=D(M,m,\mathbf{d}_q)$ and $A_p=D(M,m,\mathbf{d}_p)$, where $(\mathbf{d}_q)_i=\lambda_i/(1-\lambda_i)$ and $(\mathbf{d}_p)_i=w(\mathbf{z}_i)$.
Finally, the marginal inclusion probabilities are given by
\begin{align}\label{eq:marginal_q}
    \bar q_i = \frac{(\mathbf{d}_q)_i}{A_q} \frac{\partial A_q}{\partial(\mathbf{d}_q)_i},
\end{align}
where all derivatives can be computed using reverse-mode differentiation through ~\eqref{eq:DP_recursion}.

To address the first term of~\eqref{eq:betaSGP-elbo}, we follow~\citep{Uhrenholt2021Probabilistic} and employ the score-function estimator~\citep{williams1992simple}.
Specifically, let $\eta_i=\log\frac{\lambda_i}{1-\lambda_i}$ and define the inclusion indicator $I_i(\mathbf{Z})=\mathbbm{1}(\mathbf{z}_i\in\mathbf{Z})$.
For the cardinality-conditioned Bernoulli distribution in~\eqref{eq:q_Z}, we have
\begin{align}\label{eq:score_function}
    \frac{\partial}{\partial\eta_i} \log q_{\boldsymbol{\lambda}}(\mathbf{Z})
    = \frac{\partial}{\partial\eta_i}
    \left(
    \sum_{j=1}^{M} I_j(\mathbf{Z})\eta_j -\log A_q
    \right) = I_i(\mathbf{Z})-\bar q_i.
\end{align}
The result in~\eqref{eq:score_function} allows us to evaluate the score-function gradient estimator
\begin{align}\label{eq:REINFORCE}
    \nabla_{\boldsymbol{\eta}}
    \mathbb{E}_{q_{\boldsymbol{\lambda}}(\mathbf{Z})}
    [\mathcal{F}_1(\mathbf{Z})]
    =
    \mathbb{E}_{q_{\boldsymbol{\lambda}}(\mathbf{Z})}
    \left[
    (\mathcal{F}_1(\mathbf{Z})-b)
    \nabla_{\boldsymbol{\eta}}
    \log q_{\boldsymbol{\lambda}}(\mathbf{Z})
    \right],
\end{align}
where $b$ is a baseline computed as an exponentially weighted average of previous Monte Carlo estimates to reduce the variance of the score-function estimator~\citep{glasserman2013monte,Uhrenholt2021Probabilistic}.
We approximate this expectation via Monte Carlo sampling, drawing subsets $\mathbf{Z}^{(\ell)}\sim q_{\boldsymbol{\lambda}}(\mathbf{Z})$ of exactly $m$ inducing points using~\eqref{eq:DP_recursion}.
The ELBO term $\mathcal{F}_1$ is evaluated using SGPR or SVGP (see Appendices~\ref{app:SGPR} and~\ref{app:SVGP}).
The KL term and its gradients are computed using~\eqref{eq:KL_closed_form} and differentiation through~\eqref{eq:DP_recursion}, respectively.
At execution time, we select the $m$ candidates with the highest learned Bernoulli parameters $\lambda_i$ to form the transmitted set $\mathbf{Z}^*$.

\medskip

\noindent
\textbf{Remark 1:}
In our setting, at each timestep $t$, the candidate set $\mathcal{S}_t$ consists of all points explored by the Sensor up to time $t$ that have neither been transmitted previously nor have been observed by the Actor (i.e., the set $\mathcal{A}$ in Fig.~\ref{fig:framework}).
The transmitted inducing-point set satisfies $\mathbf{Z}_t\subseteq\mathcal{S}_t$ and $|\mathbf{Z}_t|=m_t$.
While the derivations are presented in a general form, the candidate inducing locations in our setting are restricted to the observed Sensor points due to the discretization of the mapping process.

\section{Selection of Sensor's Actions}

The Sensor selects its actions by balancing exploitation, using the RoI, with exploration, which is driven by the GP uncertainty.  
The candidate set for action selection consists of the frontier nodes, denoted by $F_t =
\left\{
\mathbf{x}\in\mathbf{X}_{1:t}^{S}
\;\middle|\;
\mathcal{N}(\mathbf{x})\cap
(\mathcal{X}\setminus\mathbf{X}_{1:t}^{S})
\neq\emptyset
\right\}$, 
where $\mathcal{N}(\mathbf{x})$ is the neighborhood of $\mathbf{x}$.

To formalize this trade-off, we construct a function inspired by the GP-UCB algorithm~\citep{Srinivas2012} and robot exploration~\citep{Asgharivaskasi2023TRO}. 
In particular, the next target point is selected as
\begin{equation}\label{eq:gp_ucb}
\bar{\mathbf{x}}_t^S 
= \arg\max_{\mathbf{x} \in F_t} 
\frac{
w_t(\mathbf{x}) 
+ \gamma\,\sigma_{*t-1}(\mathbf{x})}
{\|\mathbf{x} - \mathbf{x}_t^{S}\|},
\end{equation}
where $w_t(\mathbf{x})$ is the RoI weight, \( \sigma_{*t-1}(\mathbf{x}) = \operatorname{stdev}(\mathbf{f}_{*t-1})(\mathbf{x}) \) denotes the posterior standard deviation of the predictive model at location \( \mathbf{x} \), and \( \gamma > 0 \) controls the exploration--exploitation trade-off, and can be chosen as $\gamma \propto 
\left(
\max_{\mathbf{x}\in F_t}
w_t(\mathbf{x})
\right)
\left(
\max_{\mathbf{x}\in F_t}
\sigma_{*t-1}(\mathbf{x})
\right)^{-1}$
for normalization.

\medskip

\noindent
\textbf{Remark 2:}
The term $w_t(\mathbf{x})$ promotes task-driven behavior by biasing the Sensor toward regions that are more relevant to the Actor's objective. 
The numerator's second term in \eqref{eq:gp_ucb} promotes exploration, as it is known that greedily maximizing the GP's posterior variance is equivalent to greedily maximizing the information gain~\citep{Srinivas2012}. 

Once the target frontier point $\bar{\mathbf{x}}_t^S$ is selected, the Sensor computes its action using the unit direction vector from its current position $\mathbf{x}_t^S$ toward the target:
\begin{equation}
v_t^{S} 
=  \, \frac{\bar{\mathbf{x}}_t^S - \mathbf{x}_t^{S}}{\|\bar{\mathbf{x}}_t^S - \mathbf{x}_t^{S}\|}.
\end{equation}


\section{Experimental Results}\label{sec:result}

We conducted simulations to test the proposed algorithm on two satellite maps: a Mars slope map from HiRISE onboard the Mars Reconnaissance Orbiter~\citep{HiRISE2024} (Figure~\ref{fig:mars_sim}) and an Earth traversability map (Figure~\ref{fig:earth_sim}).
The Mars map was used to study the effect of key parameters, assuming a stationary Sensor observing the entire environment (e.g., a satellite) and selecting points for transmission using the path-agnostic RoI centered at the target location (Figure~\ref{fig:mars_env}).

The Earth traversability map was used for the second set of simulations involving both an Actor and a Sensor. 
In this scenario, both robots started simultaneously and moved one grid cell per timestep. 
We conducted 100 simulations generated by randomly varying the initial position of the Sensor $\mathbf{x}_0^{S}$ within the map.
The specifications of the robotic agents are summarized in Appendix~\ref{app:specs}.

We compared the proposed framework against six baseline approaches:
(i) Fully Informed–Fully Observable (FI-FO), where the Actor has complete knowledge of the entire environment (i.e., global map);
(ii) Fully Informed with GP (FI-GP), where all the Sensor's points are transmitted to the Actor, which reconstructs the map using a GP;
(iii) Mutual Information (MI), where the Sensor's points are selected by greedily maximizing the information gain determined by the GP kernel;
(iv) Continuous-Domain SGP (Cont-SGP), where the Sensor's point locations are optimized in the continuous domain and subsequently mapped to the nearest discrete cells;
(v) Fully Informed (FI), where all points observed by the Sensor are transmitted to the Actor without GP reconstruction; and
(vi) Uninformed (U), where no information is transmitted and the Actor relies solely on its own GP.

\subsection{Performance Metrics}

To evaluate map reconstruction, we computed the Mean Squared Error (MSE) and the Negative Log Predictive Density (NLPD) (see Appendix \ref{app:reconstruction_metrics}). 
For comparative analysis, we considered the Actor's accumulated path cost $\mathcal{C}$ and Sensor communication cost $\mathcal{B}$, excluding Actor communication as it is identical across all methods, except the path-agnostic and Uninformed approaches.
These costs are defined as $\mathcal{C} = \sum_{\mathbf{x}_i \in \pi_G} c(\mathbf{x}_i),$ $\mathcal{B} = \sum_{t=0}^{T^S} m_t,$
where $\pi_G$ denotes the sequence of cells traversed by the Actor, including repeated visits, $c(\mathbf{x}_i)$ is computed using the first case of~\eqref{eq:cell_cost}, $m_t$ is the number of transmitted points at timestep $t$, and $T^S$ is the Sensor horizon.
The average relative regret and normalized communication cost are then defined as
\begin{equation}\label{eq:ratios}
r_{\textrm{regret}} =
\frac{1}{n_{\textrm{sim}}}
\sum_{i=1}^{n_{\textrm{sim}}}
\frac{\mathcal{C}(i)-\mathcal{C}_{\textrm{min}}}
{\mathcal{C}_{\textrm{min}}}, \quad
r_{\textrm{com}} =
\frac{1}{n_{\textrm{sim}}}
\sum_{i=1}^{n_{\textrm{sim}}}
\frac{\mathcal{B}(i)}
{\mathcal{B}_{\textrm{max}}(i)},
\end{equation}
where \(n_{\textrm{sim}}\) is the total number of simulations, \(\mathcal{C}_{\textrm{min}} = \mathcal{C}_{\textrm{FI-FO}}\) corresponds to the (optimal) accumulated cost of the FI-FO framework, and \(\mathcal{B}_{\textrm{max}}\) is the communication cost of the FI-GP framework.

\subsection{Effect of \texorpdfstring{$\beta$}{beta} and the Number of Transmitted Points}

We tested the proposed $\beta$-SGP model for different values of $\beta$ and different numbers of transmitted points $m$. 
Figure~\ref{fig:mars_sim} shows the results for different values of $\beta$ and $m=60$, while additional results for a broader range of $\beta$ and $m$ values are provided in Appendix~\ref{app:ablation_study}.

We observe that increasing $\beta$ encourages the transmission of more points near the target region, thereby prioritizing the RoI. 
In contrast, smaller values of $\beta$ place greater emphasis on reducing the uncertainty of the whole reconstructed map. 
This behavior is also reflected in Table~\ref{tab:ablation_results}, where lower values of $\beta$ generally lead to smaller NLPD and MSE values.


\begin{figure*}[!t]
    \centering
    \begin{minipage}[b]{0.23\textwidth}
        \centering
        \includegraphics[width=\linewidth,
            trim=130 45 110 50, clip]{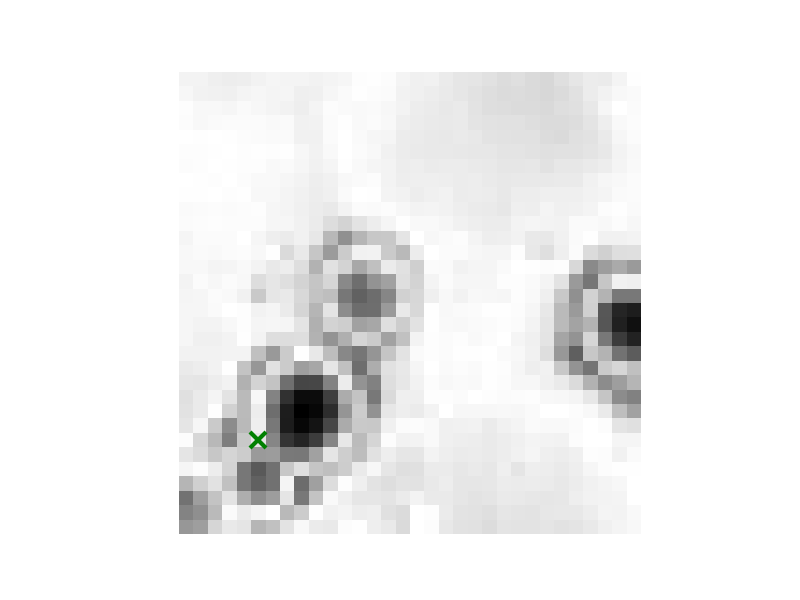}
        \caption*{(a) Mars map}
    \end{minipage}
    \hfill
    \begin{minipage}[b]{0.23\textwidth}
        \centering
        \includegraphics[width=\linewidth,
            trim=130 45 110 50, clip]{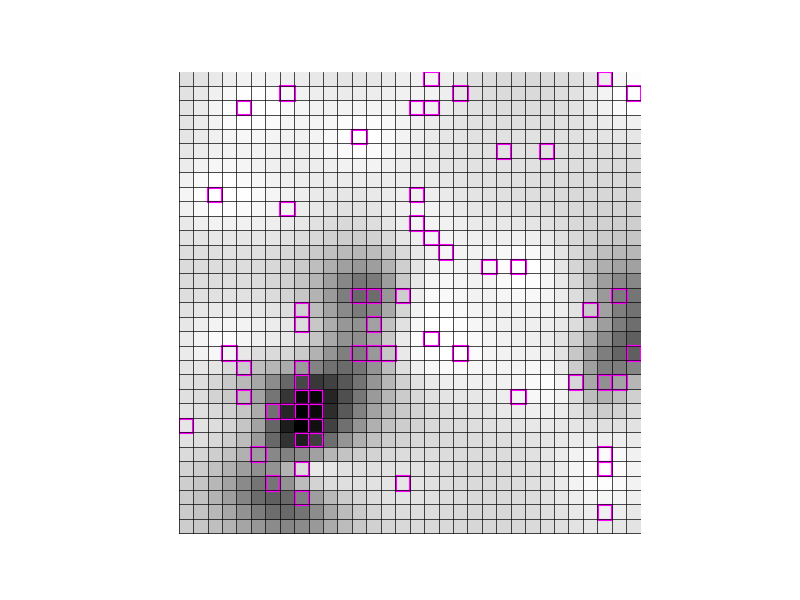}
        \caption*{(b) $\beta=1$}
    \end{minipage}
    \hfill
    \begin{minipage}[b]{0.23\textwidth}
        \centering
        \includegraphics[width=\linewidth,
            trim=130 45 110 50, clip]{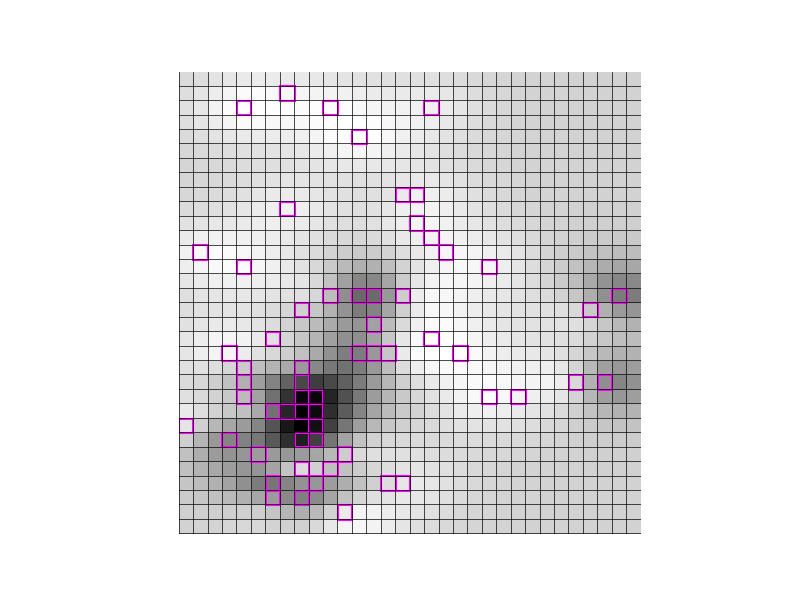}
        \caption*{(c) $\beta=200$}
    \end{minipage}
    \hfill
    \begin{minipage}[b]{0.23\textwidth}
        \centering
        \includegraphics[width=\linewidth,
            trim=130 45 110 50, clip]{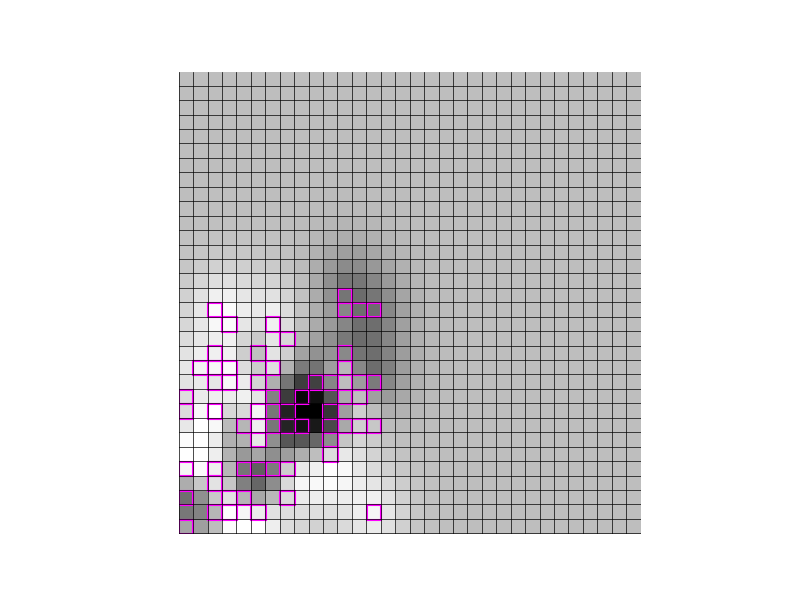}
        \caption*{(d) $\beta=2000$}
    \end{minipage}

    \caption{(a) Ground truth Mars $32 \times 32$ slope map~\citep{HiRISE2024} with goal location $\textcolor{green!50!black}{\pmb{\pmb{\times}}}$;
    (b--d) Predicted maps using $\beta$-SGP, with magenta cells indicating transmitted points and $m=60$.}
    \label{fig:mars_sim}
\end{figure*}

\begin{table}[t]
\centering
\caption{NLPD and MSE for different values of $\beta$ and $m$.}
\label{tab:ablation_results}
\small

\begin{minipage}{0.48\linewidth}
\centering
\textbf{NLPD}

\vspace{1mm}

\begin{tabular}{c|cccc}
\hline
$\beta \backslash m$ & 30 & 60 & 300 & 600 \\
\hline
1 &
\heatcell{-0.3480}{-2.0376}{-0.1544} &
\heatcell{-0.6365}{-2.0376}{-0.1544} &
\heatcell{-1.5002}{-2.0376}{-0.1544} &
\heatcell{-1.8723}{-2.0376}{-0.1544} \\
20 &
\heatcell{-0.3537}{-2.0376}{-0.1544} &
\heatcell{-0.6274}{-2.0376}{-0.1544} &
\heatcell{-1.4705}{-2.0376}{-0.1544} &
\heatcell{-1.9427}{-2.0376}{-0.1544} \\
200 &
\heatcell{-0.3070}{-2.0376}{-0.1544} &
\heatcell{-0.5465}{-2.0376}{-0.1544} &
\heatcell{-1.3800}{-2.0376}{-0.1544} &
\heatcell{-2.0376}{-2.0376}{-0.1544} \\
2000 &
\heatcell{-0.1544}{-2.0376}{-0.1544} &
\heatcell{-0.3049}{-2.0376}{-0.1544} &
\heatcell{-1.0238}{-2.0376}{-0.1544} &
\heatcell{-1.6335}{-2.0376}{-0.1544} \\
\hline
\end{tabular}
\end{minipage}
\hfill
\begin{minipage}{0.48\linewidth}
\centering
\textbf{MSE}

\vspace{1mm}

\begin{tabular}{c|cccc}
\hline
$\beta \backslash m$ & 30 & 60 & 300 & 600 \\
\hline
1 &
\heatcell{0.0231}{0.0019}{0.0470} &
\heatcell{0.0118}{0.0019}{0.0470} &
\heatcell{0.0043}{0.0019}{0.0470} &
\heatcell{0.0026}{0.0019}{0.0470} \\
20 &
\heatcell{0.0216}{0.0019}{0.0470} &
\heatcell{0.0127}{0.0019}{0.0470} &
\heatcell{0.0044}{0.0019}{0.0470} &
\heatcell{0.0023}{0.0019}{0.0470} \\
200 &
\heatcell{0.0256}{0.0019}{0.0470} &
\heatcell{0.0140}{0.0019}{0.0470} &
\heatcell{0.0053}{0.0019}{0.0470} &
\heatcell{0.0019}{0.0019}{0.0470} \\
2000 &
\heatcell{0.0470}{0.0019}{0.0470} &
\heatcell{0.0366}{0.0019}{0.0470} &
\heatcell{0.0120}{0.0019}{0.0470} &
\heatcell{0.0089}{0.0019}{0.0470} \\
\hline
\end{tabular}
\end{minipage}
\end{table}

\subsection{Comparative Analysis}

We compare the proposed $\beta$-SGP selective framework against six baselines, with the results summarized in Table~\ref{tab:comparison}.
For methods operating under bandwidth constraints (Cont-SGP, MI, and $\beta$-SGP), the number of transmitted points follows the periodic rule $m_t \in \{1,2,3,2,1,3,\ldots\}$.
Among these methods, $\beta$-SGP ($\beta=1,000$) with path-dependent RoI achieves the lowest regret of $115.61\%$, while reducing communication by $94.23\%$ compared to FI-GP.
Additionally, it reduces the path cost by $34.65\%$ compared to the U framework.
These results demonstrate the benefit of prioritizing the RoI when selecting points for transmission, compared to MI (which tends to select points more uniformly; see Appendix~\ref{app:comparative_analysis}), as well as its benefits when compared with the  other methods.

Although the path-agnostic (PA) approach exhibits higher regret, it does not require the Actor to transmit its planned path for RoI computation, making it a useful alternative to the U framework when such communication is unavailable.
Furthermore, FI-GP outperforms FI, highlighting the benefits of GP-based prediction over directly using the transmitted points.
Figure~\ref{fig:earth_sim} presents qualitative results at the final timestep, while additional results are provided in Appendix~\ref{app:comparative_analysis}.


\begin{figure}[!t]
    \centering
    \begin{minipage}[b]{0.23\linewidth}
        \centering
        \includegraphics[width=\linewidth,
            trim=130 45 110 50, clip]{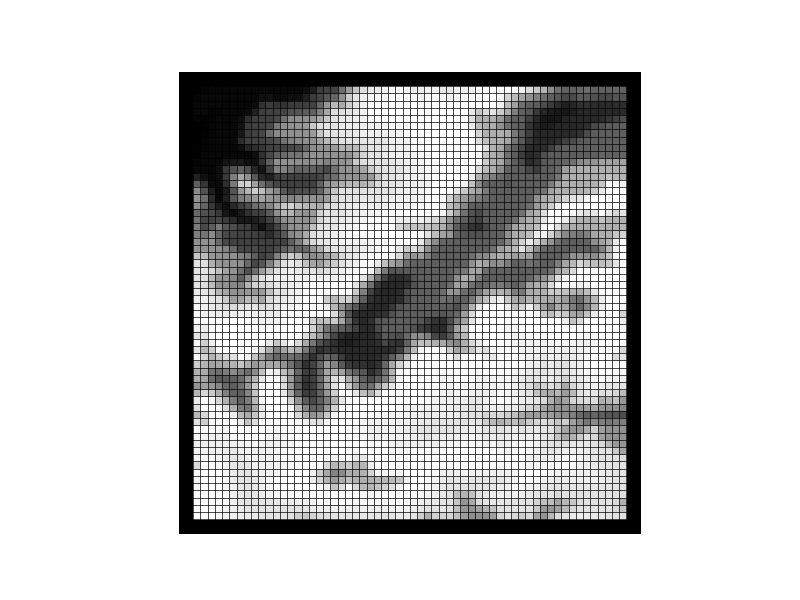}
        \caption*{(a) Earth map}
        \label{fig:earth_gt_map}
    \end{minipage}
    \hfill
    \begin{minipage}[b]{0.23\linewidth}
        \centering
        \includegraphics[width=\linewidth]{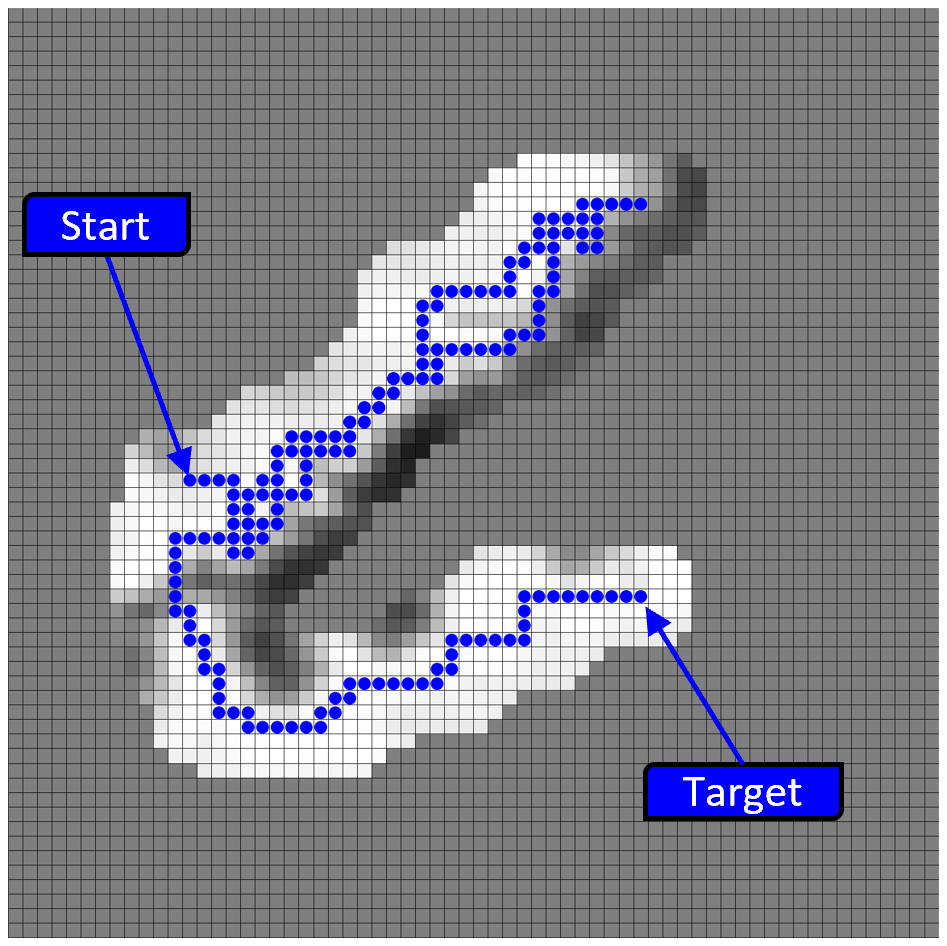}
        \caption*{(b) U}
        \label{fig:U_framework}
    \end{minipage}
    \hfill
    \begin{minipage}[b]{0.23\linewidth}
        \centering
        \includegraphics[width=\linewidth]{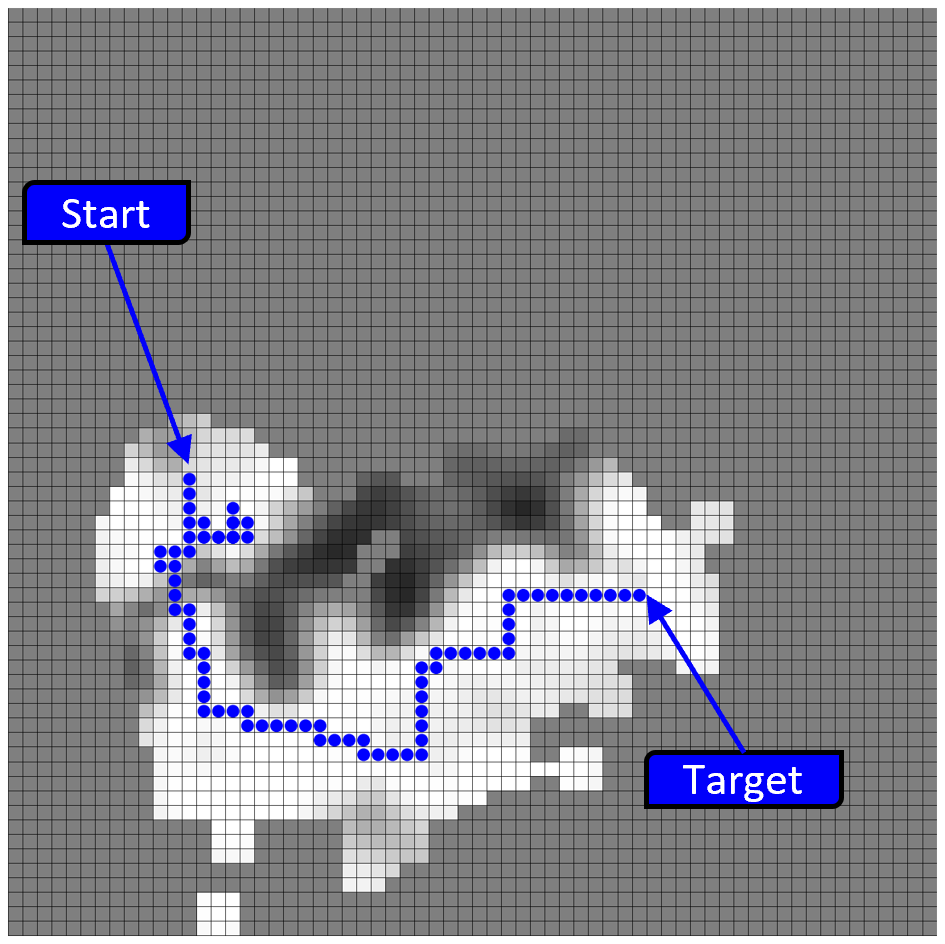}
        \caption*{(c) $\beta$-SGP ($\beta=1000$)}
        \label{fig:betaSGP_framework}
    \end{minipage}
    \hfill
    \begin{minipage}[b]{0.23\linewidth}
        \centering
        \includegraphics[width=\linewidth]{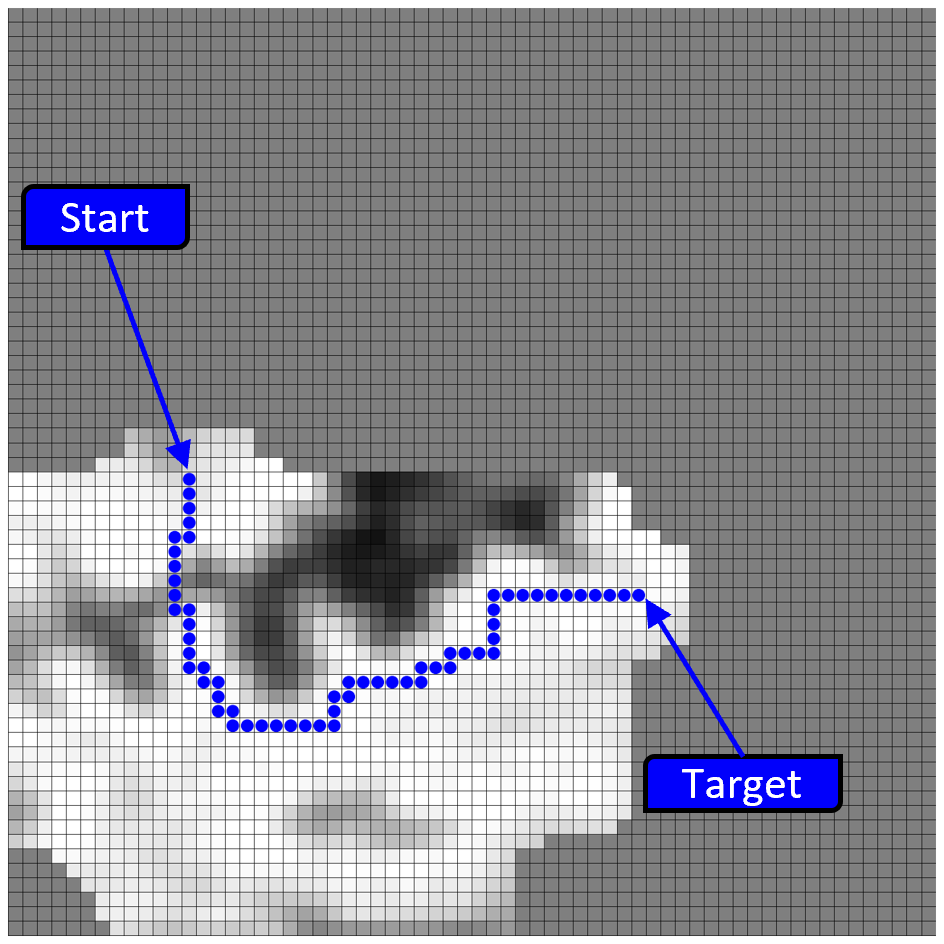}
        \caption*{(d) FI-GP}
        \label{fig:FIGP_framework}
    \end{minipage}

    \caption{(a) Earth $64 \times 64$ map;
    (b--d) Reconstructed maps obtained using
    (b) the Uninformed (U) framework,
    (c) the proposed $\beta$-SGP framework, and
    (d) the Fully Informed framework with GP (FI-GP).
    Blue and red cells indicate regions traversed by
    the Actor and the Sensor, respectively, while the
    Actor's planned path is green and $\mathbf{x}_0^{S}=(21,13)$.}
    \label{fig:earth_sim}
\end{figure}

\begin{table}[!t]
\caption{Comparative Analysis.}
\label{tab:comparison}
\centering
\small
\begin{tabular}{c|ccccc}
    \hline
    Method
    & $\mu_{t_\textrm{final}} \pm \sigma_{t_\textrm{final}}$
    & $\mu_{\mathcal{C}} \pm \sigma_{\mathcal{C}}$
    & $\mu_{\mathcal{B}} \pm \sigma_{\mathcal{B}}$
    & $r_{\textrm{regret}}$ (\%)
    & $r_\textrm{com}$ (\%) \\
    \hline

    U
    & $170$
    & $24.93$
    & $0$
    & $229.94$
    & $0.00$ \\

    Cont-SGP
    & $165.48 \pm 33.06$
    & $24.05 \pm 5.63$
    & $70.00 \pm 0.00$
    & $218.28$
    & $5.77$ \\
    
    MI
    & $148.56 \pm 49.78$
    & $19.42 \pm 6.85$
    & $70.00 \pm 0.00$
    & $157.02$
    & $5.77$ \\

    $\beta$-SGP-PA ($\beta = 1000$)
    & $166.00 \pm 29.77$
    & $24.36 \pm 5.10$
    & $70.00 \pm 0.00$
    & $222.36$
    & $5.77$ \\
    
    $\beta$-SGP ($\beta = 1$)
    & $135.67 \pm 45.29$
    & $17.85 \pm 6.43$
    & $70.00 \pm 0.00$
    & $136.29$
    & $5.77$ \\

    $\beta$-SGP ($\beta = 1000$)
    & $122.62 \pm 43.94$
    & $16.29 \pm 6.71$
    & $70.00 \pm 0.00$
    & $115.61$
    & $5.77$ \\

    FI
    & $90.86 \pm 36.54$
    & $13.36 \pm 7.11$
    & $1215.79 \pm 80.82$
    & $76.85$
    & $99.77$ \\

    FI-GP
    & $87.18 \pm 30.48$
    & $11.13 \pm 3.83$
    & $1218.88 \pm 81.62$
    & $47.32$
    & $100.00$ \\

    FI-FO
    & $58$
    & $7.56$
    & $4096$
    & $0.00$
    & {-} \\

    \hline
\end{tabular}
\end{table}

\section{Limitations}

The proposed framework focuses on discrete environments due to the Actor's grid-based navigation and the nature of the considered maps. 
Although an extension to continuous domains is possible using a Poisson point process, such an extension is not straightforward. 
Additionally, the use of an exact GP in the Actor's predictor introduces computational limitations for larger datasets. Computational efficiency could be improved by employing an SGP model and online variational conditioning (OVC)~\citep{maddox2021OVC}. 
However, in this work, we intentionally employ an exact GP for the Actor to evaluate the $\beta$-SGP framework without introducing additional approximation errors.


\section{Conclusion}
\label{sec:conclusion}

We introduced a collaborative navigation framework for heterogeneous robots operating in unknown environments under communication constraints. 
A sensor robot assists another robot to reach a target by optimizing map transmission and navigation while predicting unexplored areas.
At the core of the framework is the proposed $\beta$-Sparse Gaussian Processes model, which performs task-aware inducing point selection.
Simulation results demonstrate that the proposed framework outperforms all considered bandwidth-constrained baselines. 
Future work includes real-world validation, extension to multi-Sensor/multi-Actor settings, and improving the framework's efficiency.


\clearpage


\bibliography{references}  

\appendix

\clearpage
\section*{Appendix}

The appendix is organized as follows:

\begin{itemize}
    \item \textbf{Appendix~\ref{app:GP_all}} provides background information on Gaussian Processes (GPs), Sparse Gaussian Process Regression (SGPR), and Stochastic Variational Gaussian Processes (SVGPs).

    \item \textbf{Appendix~\ref{app:framework}} presents additional details of the proposed framework, along with the proofs of the main theoretical results.

    \item \textbf{Appendix~\ref{app:results}} describes the experimental and training specifications and includes additional experimental results.
\end{itemize}

\section{Background}\label{app:GP_all}

\subsection{Gaussian Processes (GPs)}\label{app:GP}
A Gaussian Process (GP)~\citep{Rasmussen2006} is a stochastic process in which any finite collection of random variables follows a multivariate Gaussian distribution~\citep{Rasmussen2006}.  
GPs are widely used for regression to learn a mapping from an input space $\mathcal{X} \subset \mathbb{R}^D$ to a continuous output space $\mathcal{Y} \subset \mathbb{R}$, allowing predictions at unseen inputs.
A GP is specified by its mean function $\mu(\mathbf{x})$ and covariance function (kernel) $k(\mathbf{x}, \mathbf{x}')$, and is denoted as $\mathcal{GP}\big(\mu(\mathbf{x}), k(\mathbf{x}, \mathbf{x}')\big),$
where $\mathbf{x}, \mathbf{x}' \in \mathcal{X}$, and the kernel is parameterized by hyperparameters $\boldsymbol{\theta}$. 
We model an unknown latent function $f(\mathbf{x})$ as
\[
f(\mathbf{x}) \sim \mathcal{GP}\big(\mu(\mathbf{x}), k(\mathbf{x}, \mathbf{x}')\big).
\]

In the regression setting, we are given noisy observations of the latent function at inputs $\mathbf{x}_i \in \mathcal{X}$:
\[
y_i = f(\mathbf{x}_i) + \varepsilon_i, \quad i = 1, \dots, n,
\]
where (in our case) the observation noise is assumed to be Gaussian of known variance, $\varepsilon_i \sim \mathcal{N}(0, \sigma^2)$.

Let $\mathbf{X} = [\mathbf{x}_1^\top, \dots, \mathbf{x}_n^\top]^\top \in \mathbb{R}^{n \times D}$ denote the training inputs, $\mathbf{y} = [y_1, \dots, y_n]^\top$ the corresponding observations, and 
$\mathbf{f} = [f(\mathbf{x}_1), \dots, f(\mathbf{x}_n)]^\top$ the latent function values at the training points.
Our objective is to infer the latent function at a set of test inputs $\mathbf{X}_* = [\mathbf{x}_{1*}^\top, \dots, \mathbf{x}_{\ell*}^\top]^\top.$
We denote the corresponding predictions by $\mathbf{f}_*$.

The joint prior distribution over the observations and the test function values is given by
\begin{align} \label{eq:joint_prior}
\begin{bmatrix}
    \mathbf{y} \\
    \mathbf{f}_*
\end{bmatrix}
\sim 
\mathcal{N}
\left(
    \begin{bmatrix}
        \boldsymbol{\mu} \\
        \boldsymbol{\mu}_*
    \end{bmatrix}
    ,
    \begin{bmatrix}
        K + \sigma^2 I & K_*^\top \\
        K_* & K_{**}
    \end{bmatrix}
\right), 
\end{align}
where $K = K(\mathbf{X}, \mathbf{X})$, $K_* = K(\mathbf{X}_*, \mathbf{X})$, and $K_{**} = K(\mathbf{X}_*, \mathbf{X}_*)$.

By conditioning the joint prior distribution on the observations, we can derive the conditional distribution, which is also Gaussian
\begin{align} \label{eq:posterior}
\mathbf{f}_*  |  \mathbf{X}, \mathbf{y}, \mathbf{X}_* \sim 
\mathcal{N}\big(\bar{\mathbf{f}}_*, \Sigma_* \big),
\end{align}
where
\begin{align} \label{eq:GP_posterior}
    &\bar{\mathbf{f}}_* = \mathbb{E}[\mathbf{f}_*  |  \mathbf{X}, \mathbf{y}, \mathbf{X}_*] = \boldsymbol{\mu}_* +  K_*\big[K + \sigma^2 I\big]^{-1}(\mathbf{y}- \boldsymbol{\mu}), \nonumber \\
    &\Sigma_* 
    = K_{**} - K_*\big[K + \sigma^2 I\big]^{-1}K_*^\top.
\end{align}

Inference with the above posterior distribution requires knowledge of the hyperparameters $\boldsymbol{\theta}$ of the GP's covariance function $k(\mathbf{x}, \mathbf{x}')$.
Typically, these can be estimated by maximizing the log marginal likelihood (or evidence) given by 
\begin{align} \label{eq:LML}
\log p(\mathbf{y}  |  \mathbf{X}, \theta)
&= \log \mathcal{N}\big(\mathbf{y}  |  \boldsymbol{\mu}, \, K + \sigma^2 I\big) \nonumber \\
    &= -\frac{1}{2} (\mathbf{y}- \boldsymbol{\mu})^\top \big(K + \sigma^2 I\big)^{-1} (\mathbf{y}- \boldsymbol{\mu})
  - \frac{1}{2} \log \big|K + \sigma^2 I\big| - \frac{n}{2} \log (2\pi),
\end{align}
where $ K$ depends on ${\boldsymbol{\theta}}$.

Inference using \eqref{eq:posterior} and hyperparameter optimization via \eqref{eq:LML} incur a computational complexity of $\mathcal{O}(n^3)$ and require $\mathcal{O}(n^2)$ memory, due to the matrix inversion.
This constitutes the main scalability bottleneck of GPs and naturally motivates the introduction of Sparse Gaussian Processes.

\subsection{Sparse Gaussian Process Regression (SGPR)}\label{app:SGPR}

In this section, we will describe the Sparse Gaussian Process Regression (SGPR) method as outlined in~\citep{Titsias09} which requires $\mathcal{O}(nm^2)$ computation, where $m$ is the number of inducing points.
The method proposes a variational formulation for sparse approximation that determines the kernel hyperparameters $\boldsymbol{\theta}$ and the locations of $m$ inducing points by maximizing an Evidence Lower Bound (ELBO) of the true log marginal likelihood.
The key idea is that the inducing points are treated as variational parameters and, typically $m \ll n$, thus reducing complexity.

Let $\mathbf{Z} = [\mathbf{z}_1^\top, \dots, \mathbf{z}_m^\top]^\top$ denote a set of inducing inputs, with corresponding function values $\mathbf{u} = f(\mathbf{Z}) = [f(\mathbf{z}_1), \dots, f(\mathbf{z}_m)]^\top$. 
Let the true posterior be $p(\mathbf{f}, \mathbf{u}  |  \mathbf{X}, \mathbf{y})$, which we approximate with a variational distribution $q(\mathbf{f}, \mathbf{u}) = p(\mathbf{f}  |  \mathbf{X},  \mathbf{u}) \, q(\mathbf{u})$.
For notational simplicity, we omit conditioning on hyperparameters $\boldsymbol{\theta}$ throughout this section.
The ELBO proposed in~\citep{Titsias09} is given by
\begin{align}\label{eq:elbo_titsias}
\mathcal{F}_1
= \E_{q(\mathbf{f}, \mathbf{u})}
\big[\log p(\mathbf{y}  |  \mathbf{f})\big]
- \mathrm{KL}\big(q(\mathbf{u}) \,\|\, p(\mathbf{u})\big),
\end{align}

In the special case where the likelihood \( p(\mathbf{y}  |  \mathbf{f}) \) is Gaussian, the ELBO becomes analytically tractable and is given by (assuming \( \mu(\mathbf{x}) = 0 \) for simplicity)
\begin{align}\label{eq:elbo_titsias2}
\mathcal{F}_1 = \log \mathcal{N}\big(\mathbf{y}  |  \mathbf{0}, \, Q + \sigma^2 I\big)
- \frac{1}{2\sigma^2} \operatorname{Tr}\big(K - Q\big),
\end{align}
where \( Q = K_{\mathbf{X}, \mathbf{Z}} K_{\mathbf{Z}, \mathbf{Z}}^{-1} K_{\mathbf{Z}, \mathbf{X}} \), $K_{\mathbf{X}, \mathbf{Z}} = K(\mathbf{X}, \mathbf{Z})$, and $K_{\mathbf{Z}, \mathbf{Z}} = K(\mathbf{Z}, \mathbf{Z})$.
In this case, the variational distribution is given by
\begin{align}
q^*(\mathbf{u})
&= \mathcal{N}\big(\mathbf{u}  |  \boldsymbol{m}, S\big),
\end{align}
where
$\boldsymbol{m} = \sigma^{-2} K_{\mathbf{Z}, \mathbf{Z}} A K_{\mathbf{Z}, \mathbf{X}} \mathbf{y}$, 
$S = K_{\mathbf{Z}, \mathbf{Z}} A K_{\mathbf{Z}, \mathbf{Z}}$,
and $A = \big(K_{\mathbf{Z}, \mathbf{Z}} + \sigma^{-2} K_{\mathbf{Z}, \mathbf{X}} K_{\mathbf{X}, \mathbf{Z}} \big)^{-1}$.

Note that the optimized variational parameters are required for inference, as predictions are obtained using the variational distribution over the inducing variables
\begin{align} \label{eq:SGPR_posterior}
\bar{\mathbf{f}}_{*}^q
&= K_{* \mathbf{Z}} K_{\mathbf{Z}, \mathbf{Z}}^{-1} \boldsymbol{m}, \nonumber \\
\Sigma_{*}^q
&= K_{**}
   - K_{* \mathbf{Z}} K_{\mathbf{Z}, \mathbf{Z}}^{-1} K_{\mathbf{Z} *} + K_{* \mathbf{Z}} K_{\mathbf{Z}, \mathbf{Z}}^{-1} S K_{\mathbf{Z}, \mathbf{Z}}^{-1} K_{ \mathbf{Z} *},
\end{align}
where $K_{* \mathbf{Z}} = K(X_*, \mathbf{Z})$.

\subsection{Stochastic Variational Gaussian Processes (SVGPs)}\label{app:SVGP}

In~\citep{Hensman2013BigData}, the authors propose an alternative approximation to~\eqref{eq:elbo_titsias} by decomposing the expected log-likelihood term as
\begin{align}
\mathbb{E}_{q(\mathbf{f}, \mathbf{u})}[\log p(\mathbf{y}  |  \mathbf{f})] = 
\mathbb{E}_{q(\mathbf{f}, \mathbf{u})}\left[\sum_{i=1}^{n} \log p(y_i  |  f(x_i))\right] = \sum_{i=1}^{n} \mathbb{E}_{q(f_i)}\big[\log p(y_i  |  f(x_i))\big],
\end{align}
where the variational distribution parameters $\boldsymbol{m}$ and $S$ are treated as free variables and are jointly optimized with the kernel hyperparameters using gradient-based methods. 
This stochastic variational formulation reduces the computational complexity to $\mathcal{O}(m^3)$, independent of the number of data points $n$.

\section{Framework Details}\label{app:framework}

\subsection{Remaining Components: Path Planner and Truncator}\label{app:components}

\textbf{Truncator:}
The Truncator computes the truncated map estimate \(\widehat{\mathbf{f}}_{t}\) using the Predictor's outputs, namely the GP mean and standard deviation for each cell, denoted by \(\bar{f}_*(\cdot)\) and \(\sigma_*(\cdot)\), respectively. 
Specifically, values of \(\bar{f}_*\) exceeding \(1\) are clipped to \(1\), while values below \(0\) are clipped to \(0\). 
In addition, cells with uncertainty above a threshold \(\sigma_{\mathrm{th}}\) are assigned the initial belief value \(\hat{y}_0\).

\medskip
\noindent
\textbf{Path Planner~\citep{Psomiadis2025}:}
Let $\mathcal{G}$ denote the graph associated with the discretized map $\mathcal{X}$, where each node corresponds to a grid cell and edges connect neighboring cells according to the robot action set \(U\). 
The traversal cost of a cell is defined as
\begin{equation}\label{eq:cell_cost}
    c_{t}(\mathbf{x}_i)=     
    \begin{cases}
      \widehat{f}_{t}^{(A)}(\mathbf{x}_i) + a, & \text{if } \mathbf{x}_i \in \mathcal{X}_{\epsilon},\\
      N(\epsilon + a), & \text{otherwise},
    \end{cases}
\end{equation}
where \(a\) is a constant penalty, and
$\mathcal{X}_{\epsilon} =
\left\{
\mathbf{x}_i \in \mathcal{X} :
\widehat{f}_{t}^{(A)}(\mathbf{x}_i) \leq \epsilon
\right\}$
denotes the set of feasible cells. 
The threshold \(\epsilon\) satisfies
$\hat{y}_0(\mathbf{x}_i) < \epsilon \leq 1,
\quad \forall \mathbf{x}_i \in \mathcal{X}$,
where \(\hat{y}_0\) is the initial belief value. 
Choosing \(\epsilon\) greater than the initial belief biases the Actor to explore unknown regions rather than traversing cells believed to contain obstacles.

The Actor computes the minimum-cost path using a graph search algorithm, e.g., Dijkstra's algorithm \citep{Dijkstra1959} to find
\begin{equation}\label{eq:path}
    \pi_t^* =
    \argmin_{\pi \in \Pi_t}
    \sum_{\mathbf{x}_i \in \pi}
    c_t(\mathbf{x}_i),
\end{equation}
where \(\Pi_t\) denotes the set of all paths from the Actor's current position \(\mathbf{x}_t^A\) to the goal location \(\mathbf{x}_G^A\).
Typically, the Actor executes the first cell of the computed path at each timestep as the next action. 
However, when deadlocks leading to oscillatory behavior are detected, the Actor temporarily freezes the planned path and executes multiple consecutive actions before replanning.

\subsection{Sensor Algorithm}\label{app:algorithm}

Let the Sensor state at time $t$ be $\boldsymbol{\xi}_{t} = \{ \mathbf{x}_{t}^{S}, F_{t}, \boldsymbol{\sigma}^{*S}_{t}, {w}_{t}(\cdot), \mathcal{A}_{t}, ( (\mathbf{Z}^*, \mathbf{y} (\mathbf{Z}^*))_{t}^{S}, (\mathbf{X}, \mathbf{y} (\mathbf{X}))_{0:t}^{S} \}$.
Algorithm~\ref{alg:sensor} summarizes the Sensor's procedure.
The \textsc{UpdateFrontiers} function updates the frontier set based on the Sensor's position.
The \textsc{Perception} function acquires observations from the environment at the current Sensor location.
Finally, the \textsc{Overlap} function computes the set $\mathcal{A}_t$, which contains points that have either been previously transmitted by the Sensor or independently observed by the Actor.

\begin{algorithm}[H]
\caption{Sensor Algorithm}
\label{alg:sensor}

    \textbf{Input:}
    Sensor state $\boldsymbol{\xi}_{t-1}$,
    Actor's path $\pi_t^*$
    
    \textbf{Output:}
    Updated Sensor state $\boldsymbol{\xi}_{t}$
    
    \begin{algorithmic}[1]
    
    \State $F_t \gets \textsc{UpdateFrontiers}(F_{t-1}, \mathbf{x}_{t-1}^S)$
    
    \State $\mathbf{x}_t^S \gets
    \textsc{Action}(
    \mathbf{x}_{t-1}^S,
    w_{t-1}(\cdot),
    \boldsymbol{\sigma}_{t-1}^{*S},
    F_t)$ using~\eqref{eq:gp_ucb}
    
    \State $(\mathbf{X},\mathbf{y}(\mathbf{X}))_t^S
    \gets
    \textsc{Perception}(\mathbf{x}_t^S)$
    
    \State $w_t(\cdot)
    \gets
    \textsc{RoIGenerator}(\pi_t^*)$ using~\eqref{eq:ROI_weights_simple} or~\eqref{eq:ROI_weights_path}
    
    \State $\mathcal{A}_t
    \gets
    \textsc{Overlap}(
    \mathcal{A}_{t-1},
    \mathbf{Z}_{t-1}^{*S},
    \mathbf{x}_t^A)$
    
    \State $(\mathbf{Z}^*, \mathbf{y}(\mathbf{Z}^*))_t^S
    \gets
    \textsc{$\beta$-SGP}(
    (\mathbf{X},\mathbf{y}(\mathbf{X}))_{0:t}^S,
    \mathcal{A}_t,
    w_t(\cdot))$ optimizing~\eqref{eq:betaSGP-elbo}
    
    \State $(\bar{\mathbf{f}}_*, \boldsymbol{\sigma}_t^{*S})
    \gets
    \textsc{Predictor}(
    (\mathbf{Z}^*,\mathbf{y}(\mathbf{Z}^*))_t^S,
    \mathcal{A}_t)$ using~\eqref{eq:SGPR_posterior}
    
    \State Update $\boldsymbol{\xi}_t$
    
    \State \Return $\boldsymbol{\xi}_t$

\end{algorithmic}
\end{algorithm}

\subsection{Cardinality-Conditioned Multivariate Bernoulli Distribution}\label{app:conditional_bernoulli}

Let $\mathcal{S}=\{\mathbf{z}_1,\ldots,\mathbf{z}_M\}$ denote the set of candidate inducing locations.
Assuming independent Bernoulli variables with parameters $p_i\in(0,1)$, the probability of selecting a subset $\mathbf{Z}\subseteq\mathcal{S}$ is given by
\begin{align}
    \tilde p(\mathbf Z)=
    \prod_{\mathbf z_i\in\mathbf Z}p_i
    \prod_{\mathbf z_i\notin\mathbf Z}(1-p_i).
\end{align}

Conditioning on selecting exactly $m$ inducing points yields
\begin{align} \label{eq:Cond_Bernoulli_derivation}
    p(\mathbf Z) &= \tilde p(\mathbf Z |  |\mathbf Z|=m) = 
    \frac{\tilde p(\mathbf Z, |\mathbf Z|=m)}{P_{\tilde p}(|\mathbf Z|=m)} \nonumber \\
    &= \frac{\displaystyle
    \prod_{\mathbf z_i\in\mathbf Z}p_i
    \prod_{\mathbf z_i\notin\mathbf Z}(1-p_i)}{
    \displaystyle
    \sum_{\substack{\mathbf S\subseteq\mathcal S\\|\mathbf S|=m}}\prod_{\mathbf z_i\in\mathbf S}p_i
    \prod_{\mathbf z_i\notin\mathbf S}(1-p_i)},
    \qquad |\mathbf Z|=m.
\end{align}

Note that for any arbitrary set $\mathbf S \subseteq\mathcal{S}$ (including $\mathbf Z$) we can write: 
\begin{align}
    \prod_{\mathbf z_i \in \mathbf S}p_i \prod_{\mathbf z_i \notin \mathbf S}(1-p_i) &=
    \prod_{\mathbf z_i \in \mathbf S}
    \frac{p_i}{1-p_i}
    \prod_{\mathbf z_i \in \mathbf S}(1-p_i)
    \prod_{\mathbf z_i \notin \mathbf S}(1-p_i)\nonumber \\
    &= \prod_{i=1}^M(1-p_i) \prod_{\mathbf z_i \in \mathbf S}\frac{p_i}{1-p_i}.
\end{align}

Since $\prod_{i=1}^{M}(1-p_i)$ is independent
of the selected subset, it cancels between
the numerator and denominator. 
Hence, \eqref{eq:Cond_Bernoulli_derivation} becomes
\begin{align}
    p(\mathbf Z) = \frac{ \displaystyle\prod_{\mathbf z_i \in \mathbf Z}\frac{p_i}{1-p_i}}
    { \displaystyle \sum_{\substack{\mathbf S\subseteq\mathcal S\\|\mathbf S|=m}}
\prod_{\mathbf z_i\in\mathbf S}\frac{p_i}{1-p_i}},
    \qquad |\mathbf Z|=m.
\end{align}
The probability is zero for $|\mathbf{Z}|\neq m$.

\subsection{Proof of Theorem~\ref{thm:betaSGP_ELBO}}\label{app:ThmbetaSGP_ELBO}
\begin{proof}
Consider the KL divergence between the variational distribution $q(\mathbf{f}, \mathbf{u}, \mathbf{Z})$ and the true posterior $p(\mathbf{f}, \mathbf{u}, \mathbf{Z}  |  \mathbf{y}, \mathbf{X})$
\begin{align}
&\mathrm{KL}\big(q(\mathbf{f}, \mathbf{u}, \mathbf{Z}) \,\|\, p(\mathbf{f}, \mathbf{u}, \mathbf{Z}  |   \mathbf{y}, \mathbf{X}) \big) = 
\sum_{\mathbf{Z}} \int q(\mathbf{f}, \mathbf{u}, \mathbf{Z})
    \log \frac{q(\mathbf{f}, \mathbf{u}, \mathbf{Z})}
              {p(\mathbf{f}, \mathbf{u}, \mathbf{Z}  |   \mathbf{y}, \mathbf{X})}
    \, d\mathbf{f} \, d\mathbf{u}  \nonumber \\
    &=  \sum_{\mathbf{Z}}\int q(\mathbf{f}, \mathbf{u}, \mathbf{Z}) 
    \big( - \log p(\mathbf{y}  |  \mathbf{f}, \mathbf{u}, \mathbf{Z}, \mathbf{X} ) + 
    \log \frac{q(\mathbf{f}, \mathbf{u}, \mathbf{Z})}{p(\mathbf{f}, \mathbf{u}, \mathbf{Z}  |  \mathbf{X})} + \log p(\mathbf{y}  |  \mathbf{X}) \big)
    \, d\mathbf{f} \, d\mathbf{u}  \nonumber \\
    &=  \log p(\mathbf{y}  |  \mathbf{X}) +   \sum_{\mathbf{Z}} \int q(\mathbf{f}, \mathbf{u}, \mathbf{Z}) 
    \big( - \log p(\mathbf{y}  |  \mathbf{f}) + 
    \log \frac{ q(\mathbf{u}  |  \mathbf{Z}) \, q(\mathbf{Z})}{ p(\mathbf{u}  |  \mathbf{Z}) \, p(\mathbf{Z})} \big)
    \, d\mathbf{f} \, d\mathbf{u},
\end{align}
where the second equality follows from Bayes’ rule, and the third uses the conditional independence of $\mathbf{y}$ and $(\mathbf{u},\mathbf{Z})$ given $\mathbf{f}$, and the assumed factorizations of the joint prior and variational distributions.
Rearranging terms yields:
\begin{align} \label{eq:mll}
\log p(\mathbf{y}  |  \mathbf{X}) = \mathrm{KL}\big(q(\mathbf{f}, \mathbf{u}, \mathbf{Z}) \,\|\, p(\mathbf{f}, \mathbf{u}, \mathbf{Z}  |   \mathbf{y}, \mathbf{X}) \big) + \E_{q(\mathbf{Z})} \big[ \mathcal{F}_1 \big] - 
    \mathrm{KL}\big(q(\mathbf{Z}) \,\|\, p(\mathbf{Z}) \big),
\end{align}
where $\mathcal{F}_1 = \E_{q(\mathbf{f}, \mathbf{u}  |  \mathbf{Z})} [\log p(\mathbf{y}  |  \mathbf{f})] - \mathrm{KL}\big(q(\mathbf{u}  |  \mathbf{Z}) \,\|\, p(\mathbf{u}  |  \mathbf{Z}))$.
Since the KL divergence is non-negative, the standard variational lower bound is
\begin{align}
    \mathcal{F}_{\mathrm{ELBO}} = \mathbb{E}_{q(\mathbf{Z})}[\mathcal{F}_1] - \mathrm{KL}\big(q(\mathbf{Z})\,\|\,p(\mathbf{Z})\big).
\end{align}

Incorporating the constraint in~\eqref{eq:KL_constraint} through a Lagrangian relaxation gives
\begin{align}
    \mathcal{L} ={}&
    \mathcal{F}_{\mathrm{ELBO}} - \tilde{\beta}
    \big( \mathrm{KL}(q(\mathbf{Z})\,\|\,p(\mathbf{Z})) -\alpha \big)
    \nonumber\\
    ={}& \mathbb{E}_{q(\mathbf{Z})}[\mathcal{F}_1] - (1+\tilde{\beta})
    \mathrm{KL}\big(q(\mathbf{Z})\,\|\,p(\mathbf{Z})\big) + \tilde{\beta}\alpha.
\end{align}

Defining $\beta=1+\tilde{\beta}\geq 1$ and dropping the constant \(\tilde{\beta}\alpha\) with respect to the variational and GP parameters yields \(\mathcal F_2\) in~\eqref{eq:betaSGP-elbo}.
Furthermore, from~\eqref{eq:mll},
\begin{align}
    \log p(\mathbf{y} | \mathbf{X})-\mathcal{F}_2
    =
    \mathrm{KL}\big(
    q(\mathbf{f},\mathbf{u},\mathbf{Z})
    \,\|\,
    p(\mathbf{f},\mathbf{u},\mathbf{Z}
     | \mathbf{y},\mathbf{X})
    \big) +
    (\beta-1)
    \mathrm{KL}\big(
    q(\mathbf{Z})\,\|\,p(\mathbf{Z})
    \big)
    \geq 0.
\end{align}
The inequality follows from the non-negativity of KL divergence and $\beta\geq 1$. Therefore, $\mathcal{F}_2$ is a lower bound on the log marginal likelihood.
\end{proof}

\subsection{Proof of Lemma~\ref{lemma:KL_ELBO}}
\label{app:KL_ELBO}

\begin{proof}
From~\eqref{eq:p_z} and~\eqref{eq:q_Z}, the prior and variational distributions satisfy
\begin{align}
    p(\mathbf{Z}) =
    \frac{1}{A_p}
    \prod_{\mathbf{z}_i\in\mathbf{Z}}
    w(\mathbf{z}_i), \quad
    q_{\boldsymbol{\lambda}}(\mathbf{Z}) =
    \frac{1}{A_q}
    \prod_{\mathbf{z}_i\in\mathbf{Z}}
    \frac{\lambda_i}{1-\lambda_i},
\end{align}
for $|\mathbf{Z}|=m$.
The KL divergence between  $p(\mathbf{Z})$ and $ q_{\boldsymbol{\lambda}}(\mathbf{Z})$  is
\begin{align}
    \mathrm{KL}\big(
    q_{\boldsymbol{\lambda}}(\mathbf{Z})
    \,\|\,p(\mathbf{Z})\big) =
    \sum_{\substack{\mathbf{Z}\subseteq \mathcal S_t\\
    |\mathbf{Z}|=m}}
    q_{\boldsymbol{\lambda}}(\mathbf{Z})
    \log
    \frac{q_{\boldsymbol{\lambda}}(\mathbf{Z})}
    {p(\mathbf{Z})} =
    \mathbb{E}_{q_{\boldsymbol{\lambda}}}
    \left[
    \sum_{\mathbf{z}_i\in\mathbf{Z}}
    \log
    \frac{\lambda_i}
    {(1-\lambda_i)w(\mathbf{z}_i)}
    \right]
    +\log\frac{A_p}{A_q}.
\end{align}

Let $\bar q_i = \mathbb{P}_{q_{\boldsymbol{\lambda}}} (\mathbf{z}_i\in\mathbf{Z})$ denote the marginal inclusion probability under the cardinality-conditioned variational distribution, and let $I_i(\mathbf{Z})$ be the indicator function that specifies whether the candidate inducing point \(\mathbf z_i\) belongs to the selected subset \(\mathbf Z\).
By linearity of expectation,
\begin{align}
    \mathbb{E}_{q_{\boldsymbol{\lambda}}}
    \left[
    \sum_{\mathbf{z}_i\in\mathbf{Z}}
    \log
    \frac{\lambda_i}
    {(1-\lambda_i)w(\mathbf{z}_i)}
    \right] &=
    \mathbb{E}_{q_{\boldsymbol{\lambda}}}
    \left[
    \sum_{i=1}^{M} I_i(\mathbf{Z})
    \log\frac{\lambda_i}
    {(1-\lambda_i)w(\mathbf{z}_i)}
    \right] \nonumber\\
    &=
    \sum_{i=1}^{M}\bar q_i
    \left[
    \log\frac{\lambda_i}{1-\lambda_i}
    -\log w(\mathbf{z}_i)
    \right],
\end{align}
where the last equality follows from $\mathbb{E}_{q_{\boldsymbol{\lambda}}} [I_i(\mathbf{Z})]=\bar q_i$.
Substituting into the KL divergence yields
\begin{align}
    \mathrm{KL}\big(
    q_{\boldsymbol{\lambda}}(\mathbf{Z})
    \,\|\,p(\mathbf{Z})\big) =
    \sum_{i=1}^{M}\bar q_i
    \left[
    \log\frac{\lambda_i}{1-\lambda_i}
    -\log w(\mathbf{z}_i)
    \right]
    +\log\frac{A_p}{A_q},
\end{align}
which completes the proof.
\end{proof}

\clearpage
\section{Specifications and Additional Results}\label{app:results}

\subsection{Computer and Training Specifications}\label{app:specs}

For online Sensor training, the SGP criterion in \eqref{eq:SGP-elbo} was optimized for up to 200 epochs, followed by up to 100 additional epochs for the $\beta$-SGP criterion in \eqref{eq:betaSGP-elbo}.
For online Actor training, the model was initially trained for up to 200 epochs and subsequently updated at each timestep for up to 50 epochs, initialized with the previously learned parameters to reduce computational cost.
These values represent the maximum number of training epochs, with convergence typically achieved earlier.
For the Sensor's Predictor optimizing \eqref{eq:SGP-elbo}, we used the Adam optimizer~\citep{kingma2015adam} with a learning rate of 0.02.
For the score function estimator, we again used the Adam optimizer with a learning rate of 0.1 and 16 Monte Carlo samples.
For the Actor's Predictor, we used the Adam optimizer with a learning rate of 0.05, combined with a cosine annealing learning rate scheduler.

A fixed random seed was used to ensure reproducibility of the results. 
All experiments were conducted on a computer equipped with an 11th Gen Intel Core i7-11800H CPU (2.3 GHz, 8 cores) and 16 GB of RAM. 
The models were implemented using PyTorch v.2.5.1 \citep{paszke2019pytorch}.

\begin{table}[H]
\centering
\caption{Robot Specifications.}
\label{tab:spec_sim2}
\begin{tabular}{c|c}
\hline
Actor's local map & $5 \times 5$ grid \\
\hline
Sensor's local map & $23 \times 23$ grid \\
\hline
Robots' action set $U$ & \{UP, DOWN, LEFT, RIGHT\} \\
\hline
Actor's start $\mathbf{x}_0^{A}$ & $(12, 33)$ \\
\hline
Actor's goal $\mathbf{x}_G^{A}$ & $(43, 25)$ \\
\hline
Sensor's horizon $T^S$ & $35$ \\
\hline
Initial belief value \(\hat{y}_0\) & 0.5 \\
\hline
Path planner threshold \(\epsilon\) & 0.501 \\
\hline
Path planner penalty \(a\) & 0.1 \\
\hline
Actor's perception noise \(\sigma^{A}\) & 0.01 \\
\hline
Sensor's perception noise \(\sigma^{S}\) & 0.05 \\
\hline
Truncator's threshold \(\sigma_{\mathrm{th}}\) & 0.14 \\
\hline
RoI uncertainty $\tilde{\sigma}$ & $\frac{\|\mathbf{x}_0^A - \mathbf{x}_G^{A}\|}{3}$ \\
\hline
Number of Sensor transmitted points $m_t^*$ & $\{1,2,3,2,1,3,\dots\}$ \\
\hline
$\gamma$ & $0.05\frac{\max_{\mathbf{x}\in F_t}
w_t(\mathbf{x})}
{\max_{\mathbf{x}\in F_t}
\sigma_{*t-1}(\mathbf{x})
}$ \\
\hline
\end{tabular}
\end{table}

\subsection{Map Reconstruction Metrics}
\label{app:reconstruction_metrics}

To evaluate map reconstruction accuracy, we use the Mean Squared Error (MSE) and Negative Log Predictive Density (NLPD), defined as
\begin{align}
    \mathrm{MSE}
    &=
    \frac{1}{N}\sum_{i=1}^{N}
    \left(f(\mathbf{x}_i)-\bar{f}_*(\mathbf{x}_i)\right)^2,
    \label{eq:MSE}\\
    \mathrm{NLPD}
    &=
    \frac{1}{N}\sum_{i=1}^{N}
    \left[
    \frac{\left(f(\mathbf{x}_i)-\bar{f}_*(\mathbf{x}_i)\right)^2}{2\sigma_i^2}
    +\frac{1}{2}\log(2\pi\sigma_i^2)
    \right],
    \label{eq:NLPD}
\end{align}
where \(f(\mathbf{x}_i)\) is the ground-truth value at location \(\mathbf{x}_i\), while $\bar{f}_*(\mathbf{x}_i)$ and $\sigma_i^2$ denote the GP predictive mean and variance at cell $\mathbf{x}_i$, respectively.
While MSE measures reconstruction accuracy, NLPD additionally accounts for predictive uncertainty.
Lower values indicate better performance for both metrics.

\subsection{Qualitative Results for Ablation Study}\label{app:ablation_study}

\begin{figure}[h]
\centering

\setlength{\tabcolsep}{2pt}

\begin{tabular}{c c c c c}

& $m=30$ & $m=60$ & $m=300$ & $m=600$ \\

\raisebox{0.07\linewidth}{$\beta=1$}
&
\includegraphics[width=0.18\linewidth, trim=130 45 110 50, clip]{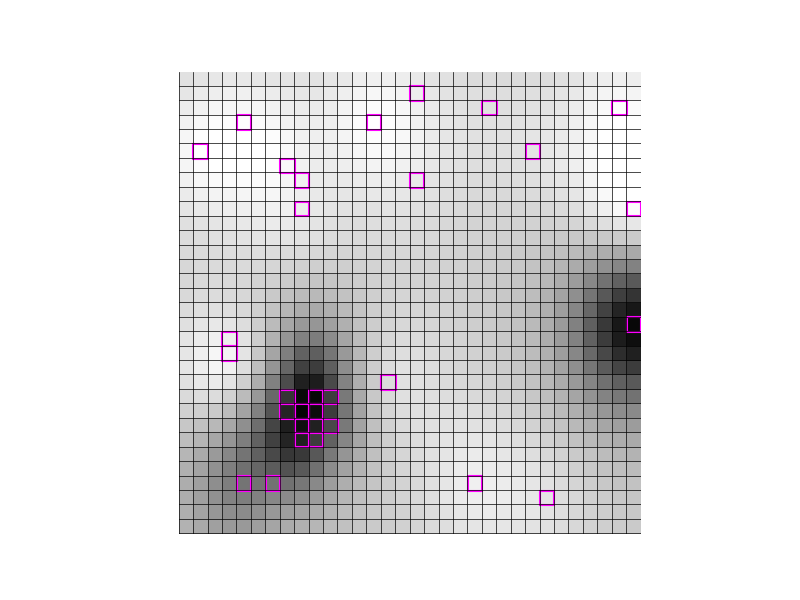}
&
\includegraphics[width=0.18\linewidth, trim=130 45 110 50, clip]{./figs/sim1/beta1_points60.png}
&
\includegraphics[width=0.18\linewidth, trim=130 45 110 50, clip]{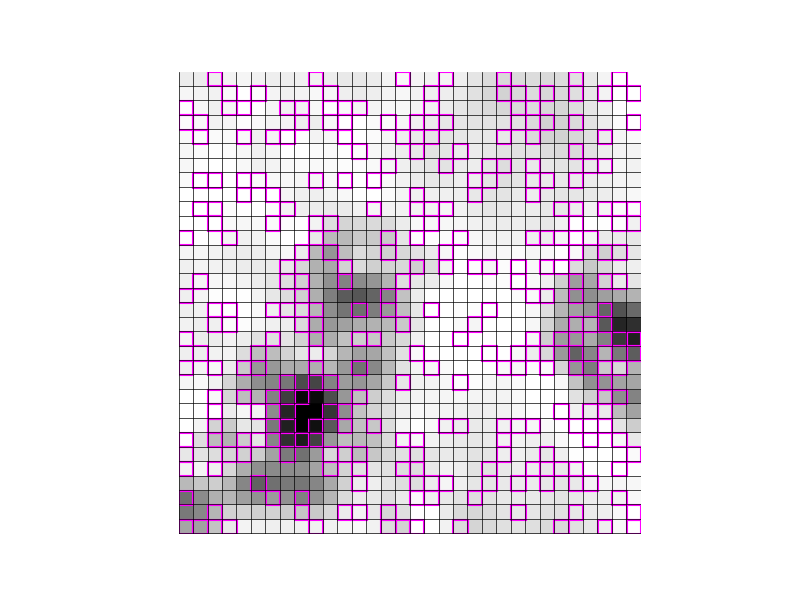}
&
\includegraphics[width=0.18\linewidth, trim=130 45 110 50, clip]{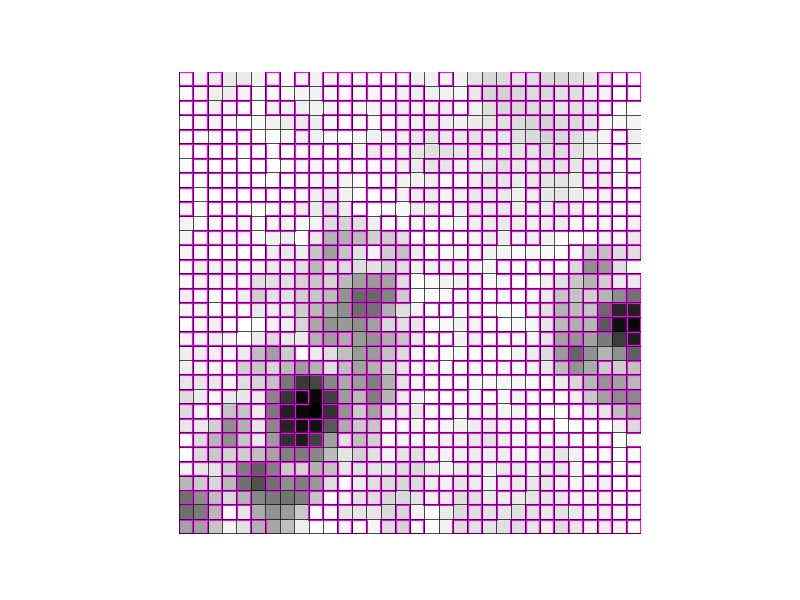}
\\

\raisebox{0.07\linewidth}{$\beta=20$}
&
\includegraphics[width=0.18\linewidth, trim=130 45 110 50, clip]{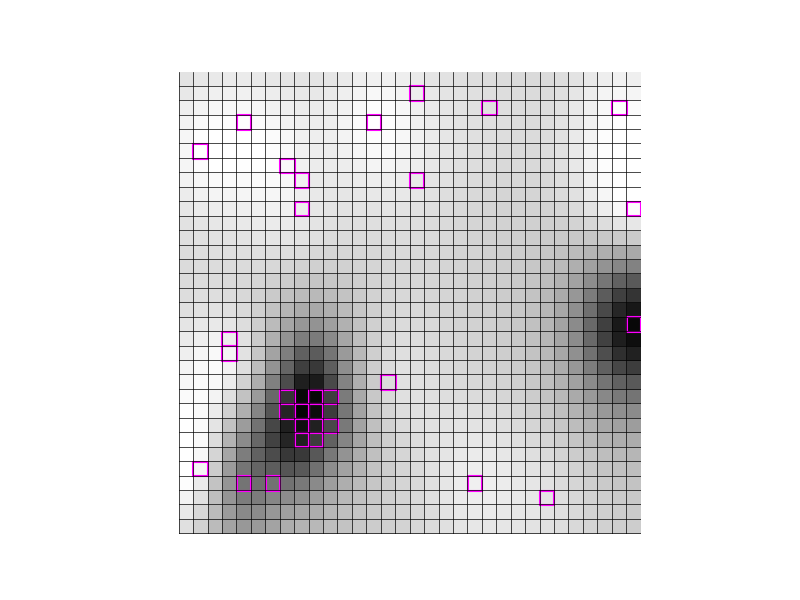}
&
\includegraphics[width=0.18\linewidth, trim=130 45 110 50, clip]{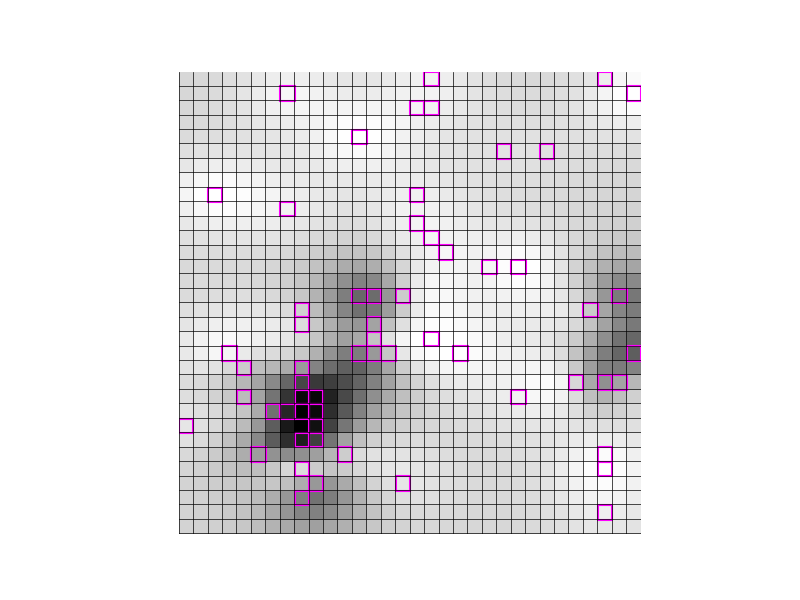}
&
\includegraphics[width=0.18\linewidth, trim=130 45 110 50, clip]{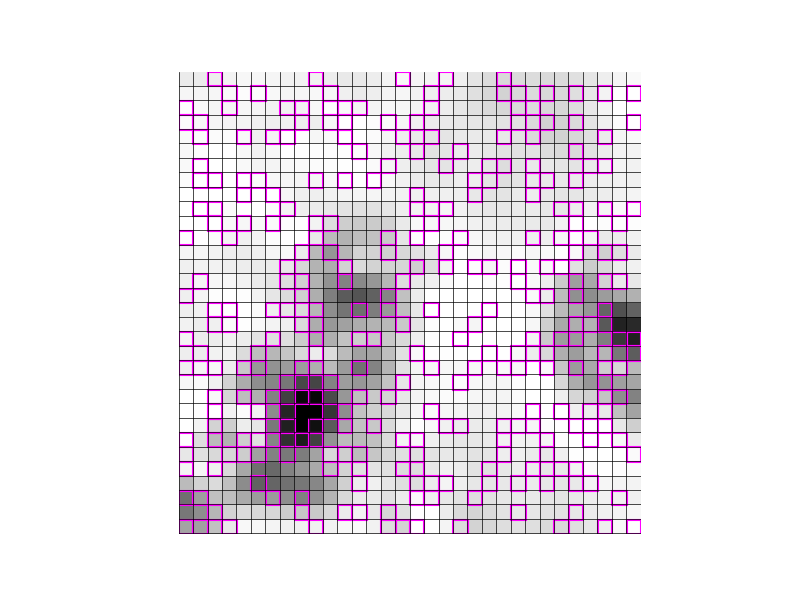}
&
\includegraphics[width=0.18\linewidth, trim=130 45 110 50, clip]{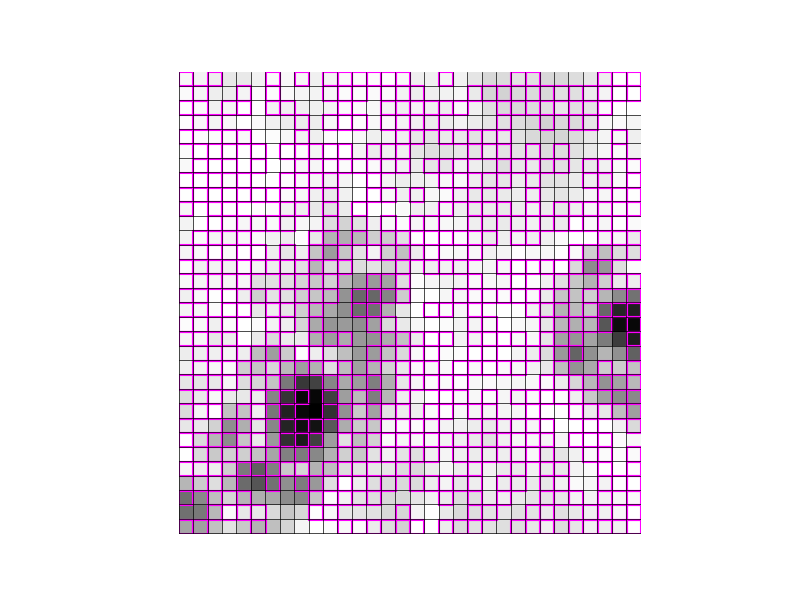}
\\

\raisebox{0.07\linewidth}{$\beta=200$}
&
\includegraphics[width=0.18\linewidth, trim=130 45 110 50, clip]{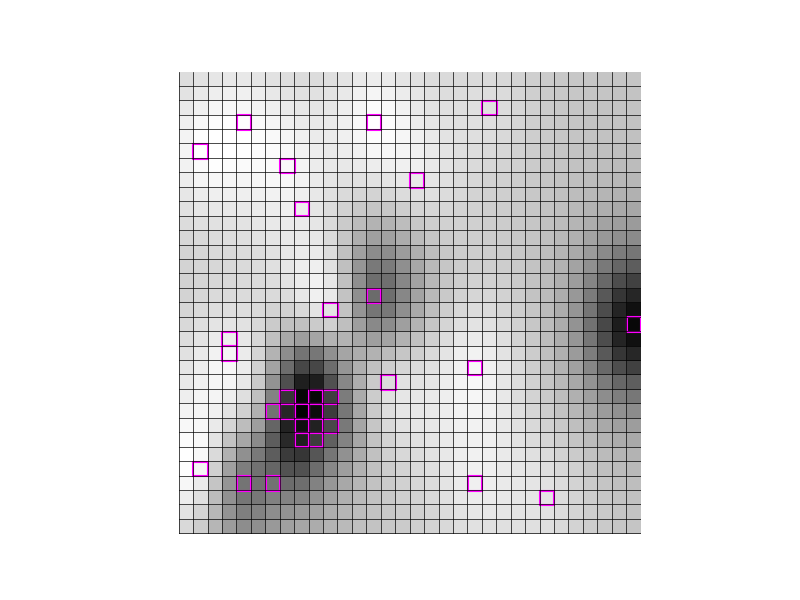}
&
\includegraphics[width=0.18\linewidth, trim=130 45 110 50, clip]{./figs/sim1/beta200_points60.png}
&
\includegraphics[width=0.18\linewidth, trim=130 45 110 50, clip]{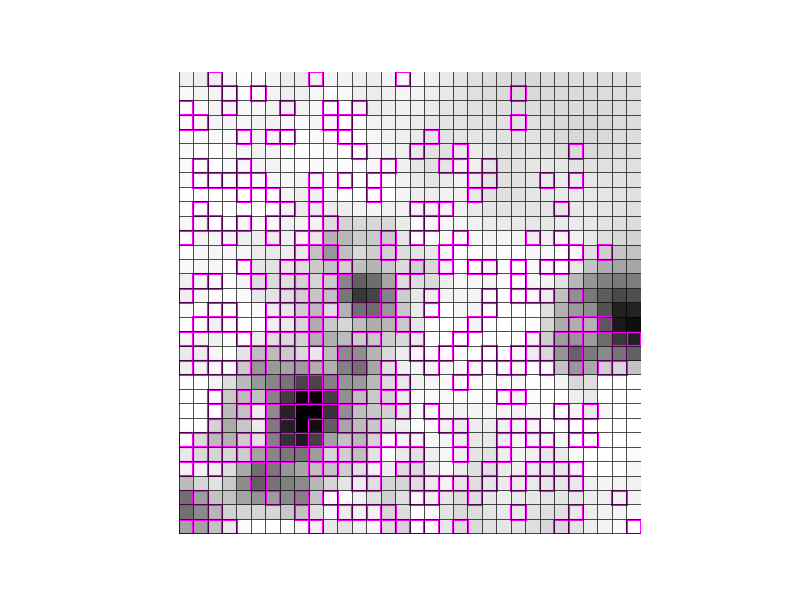}
&
\includegraphics[width=0.18\linewidth, trim=130 45 110 50, clip]{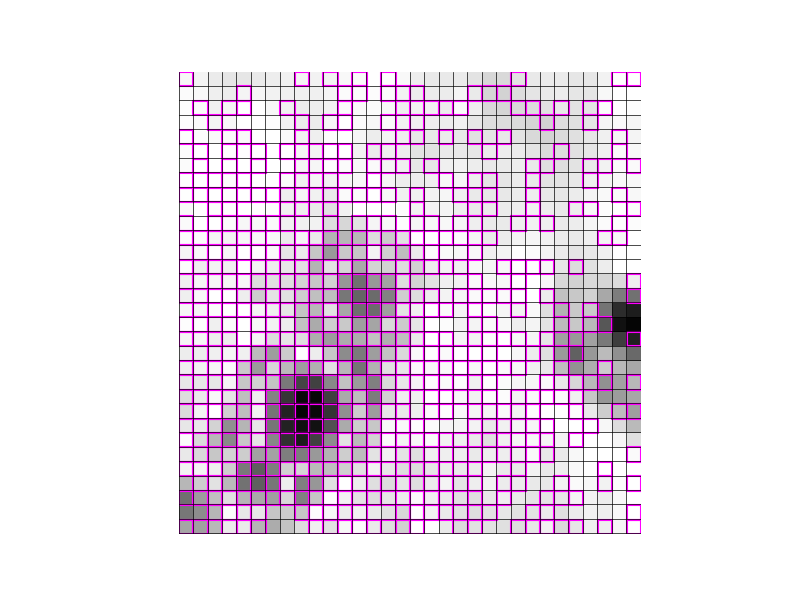}
\\

\raisebox{0.07\linewidth}{$\beta=2000$}
&
\includegraphics[width=0.18\linewidth, trim=130 45 110 50, clip]{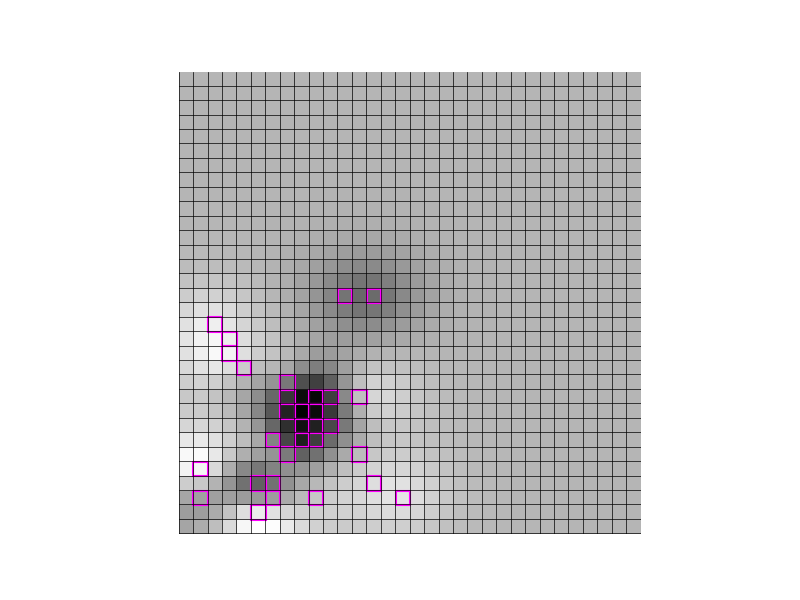}
&
\includegraphics[width=0.18\linewidth, trim=130 45 110 50, clip]{./figs/sim1/beta2000_points60.png}
&
\includegraphics[width=0.18\linewidth, trim=130 45 110 50, clip]{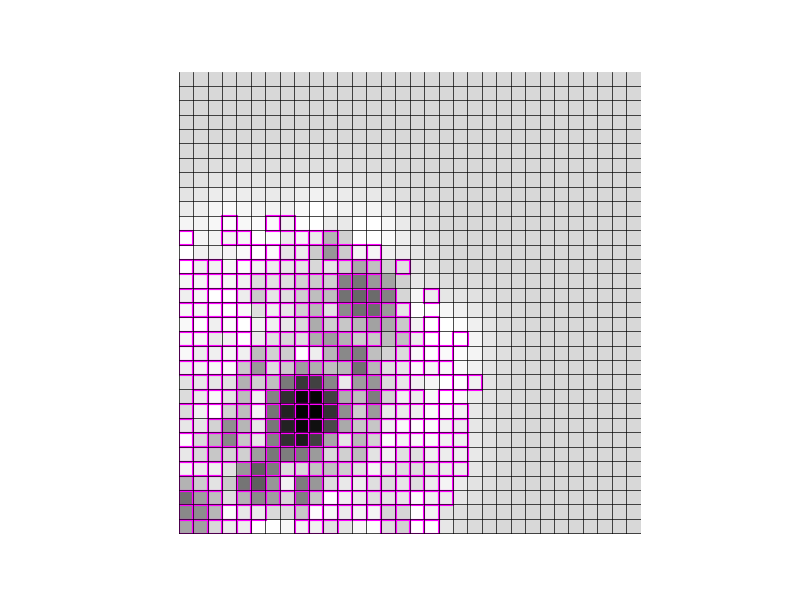}
&
\includegraphics[width=0.18\linewidth, trim=130 45 110 50, clip]{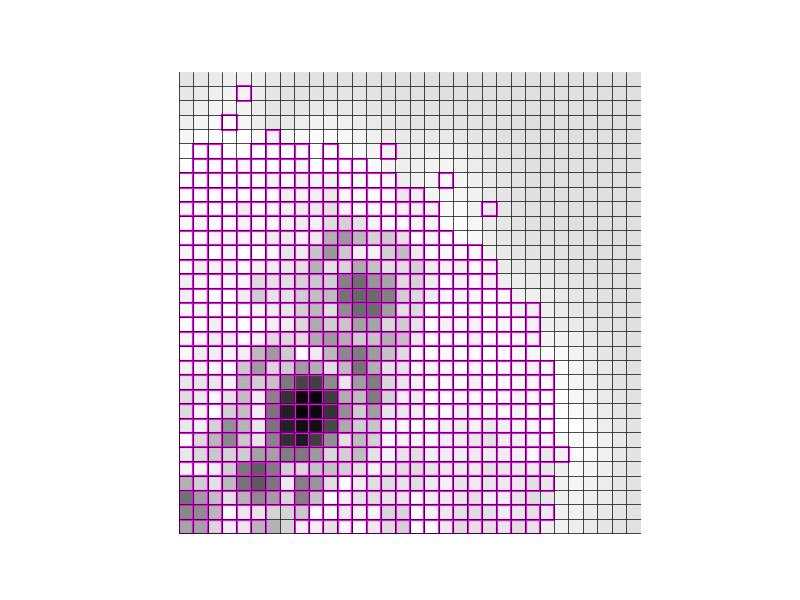}
\\

\end{tabular}

\caption{Predicted maps using $\beta$-SGP for different values of $\beta$ and number of transmitted points $m$.
Columns correspond to increasing communication budget $m$, while rows correspond to different values of $\beta$.
Magenta cells indicate transmitted points.
}
\label{fig:mars_sim_extended}

\end{figure}

\clearpage

\subsection{Qualitative Results for Comparative Analysis}\label{app:comparative_analysis}

\begin{figure}[H]
    \centering
    \setlength{\tabcolsep}{2pt}
    \begin{tabular}{c c c c}
        & $t=20$ & $t=35$ & $t_\textrm{final}$ \\

        \raisebox{0.12\linewidth}{$U$}
        &
        \includegraphics[width=0.23\linewidth]{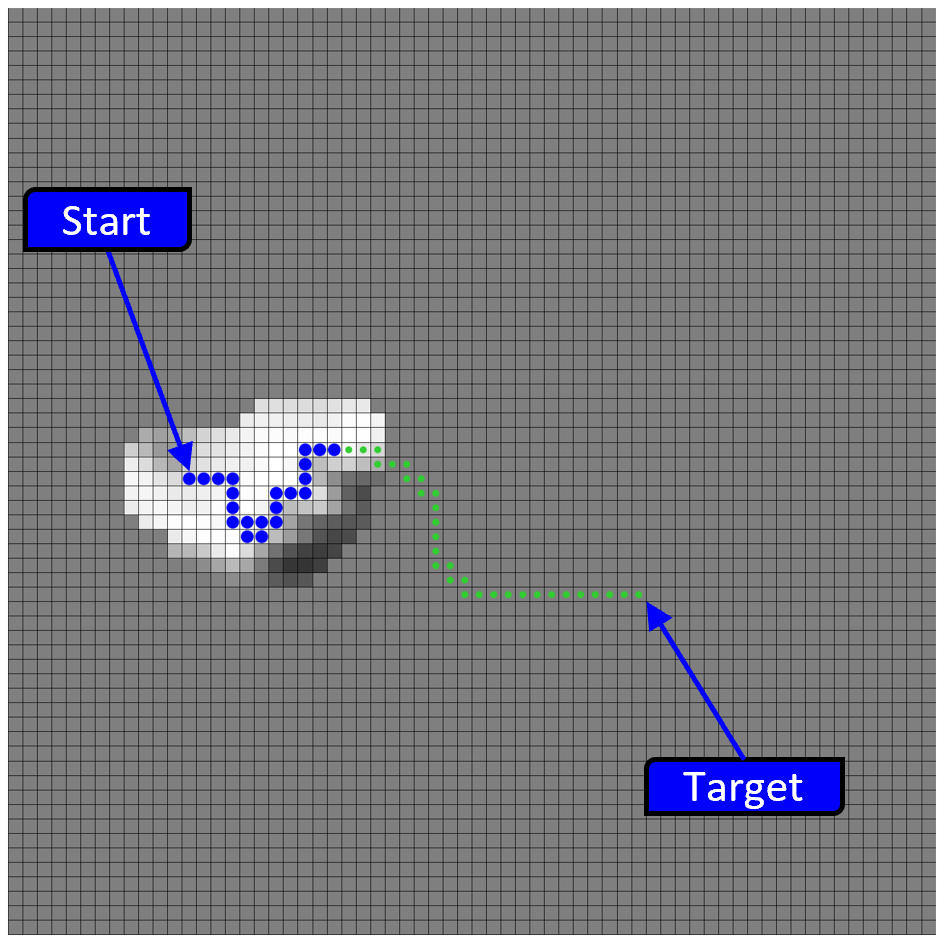}
        &
        \includegraphics[width=0.23\linewidth]{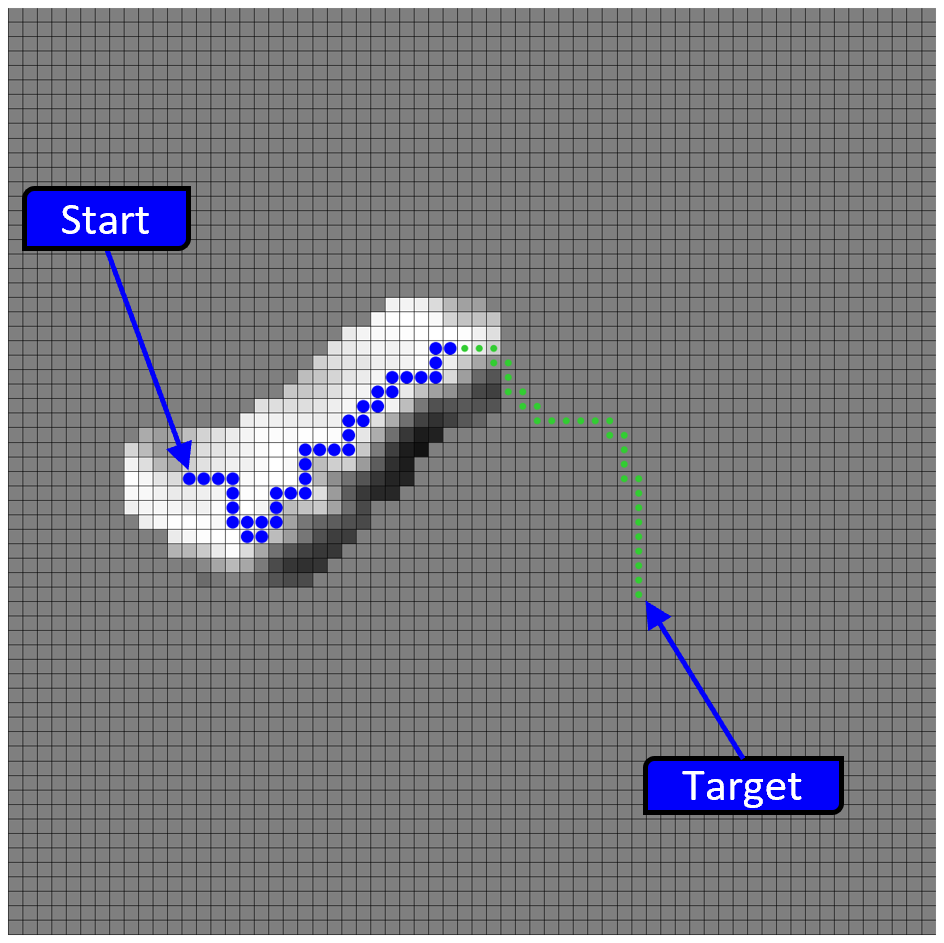}
        &
        \includegraphics[width=0.23\linewidth]{./figs/sim2/U_final.png}
        \\

        \raisebox{0.12\linewidth}{Cont-SGP}
        &
        \includegraphics[width=0.23\linewidth]{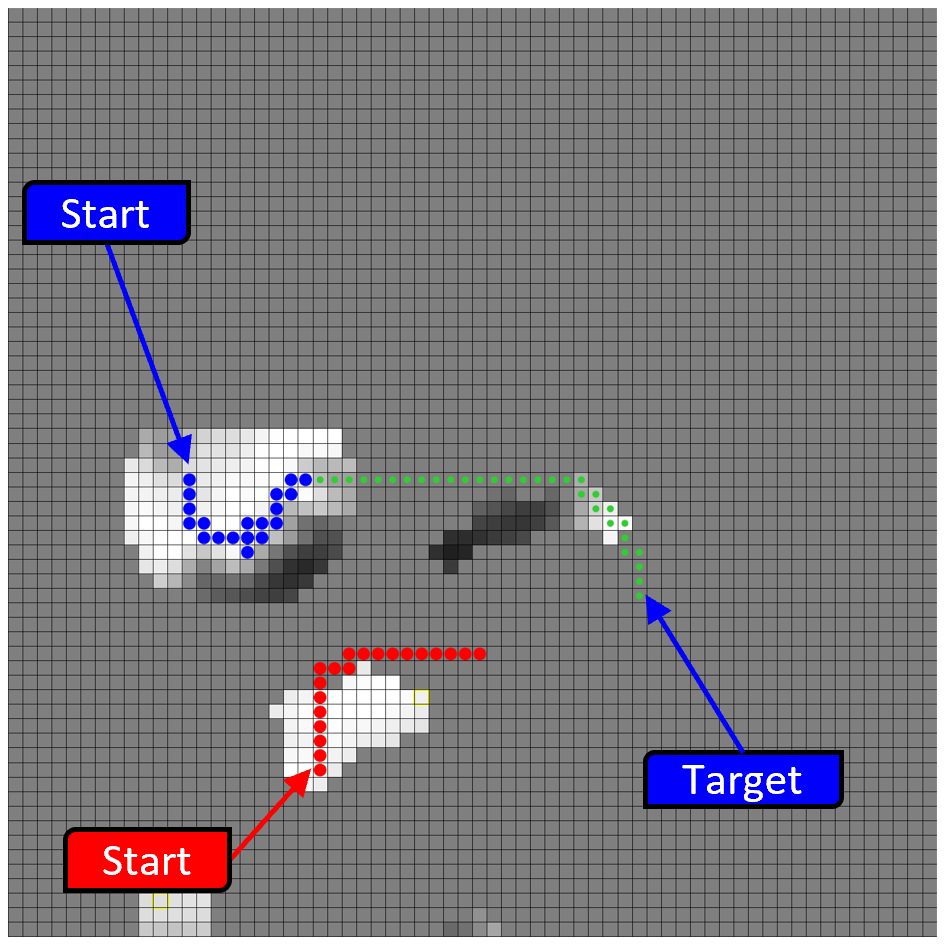}
        &
        \includegraphics[width=0.23\linewidth]{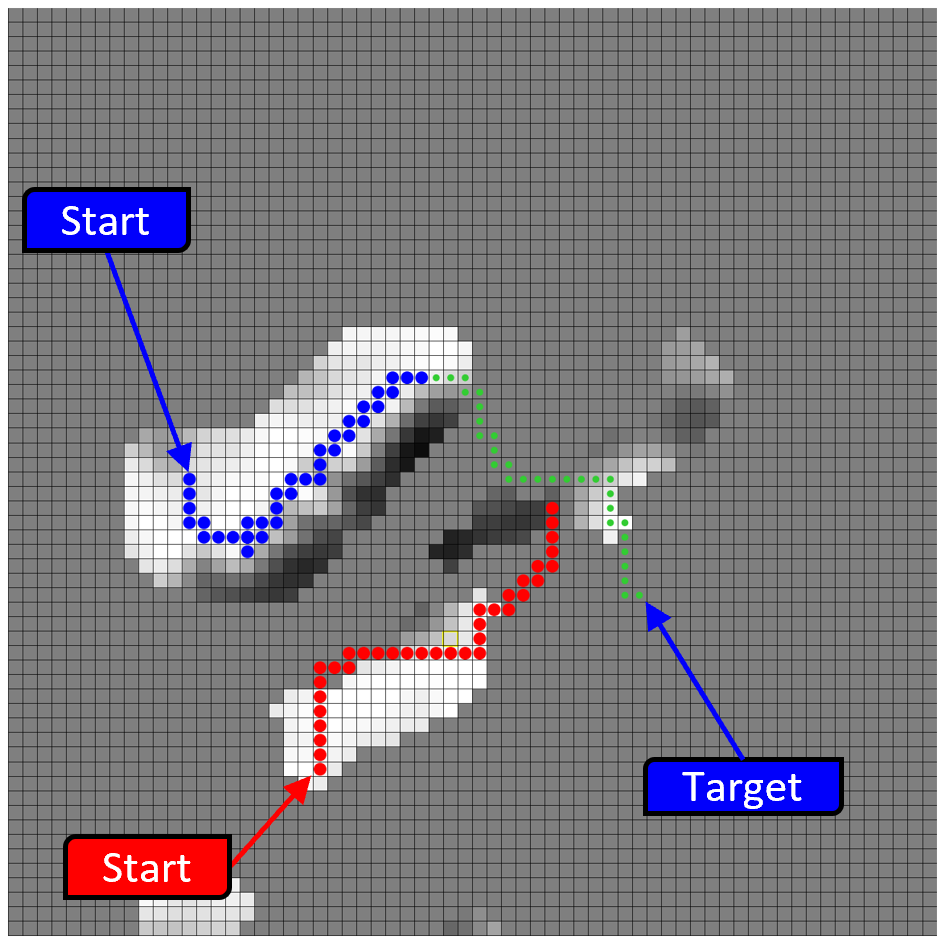}
        &
        \includegraphics[width=0.23\linewidth]{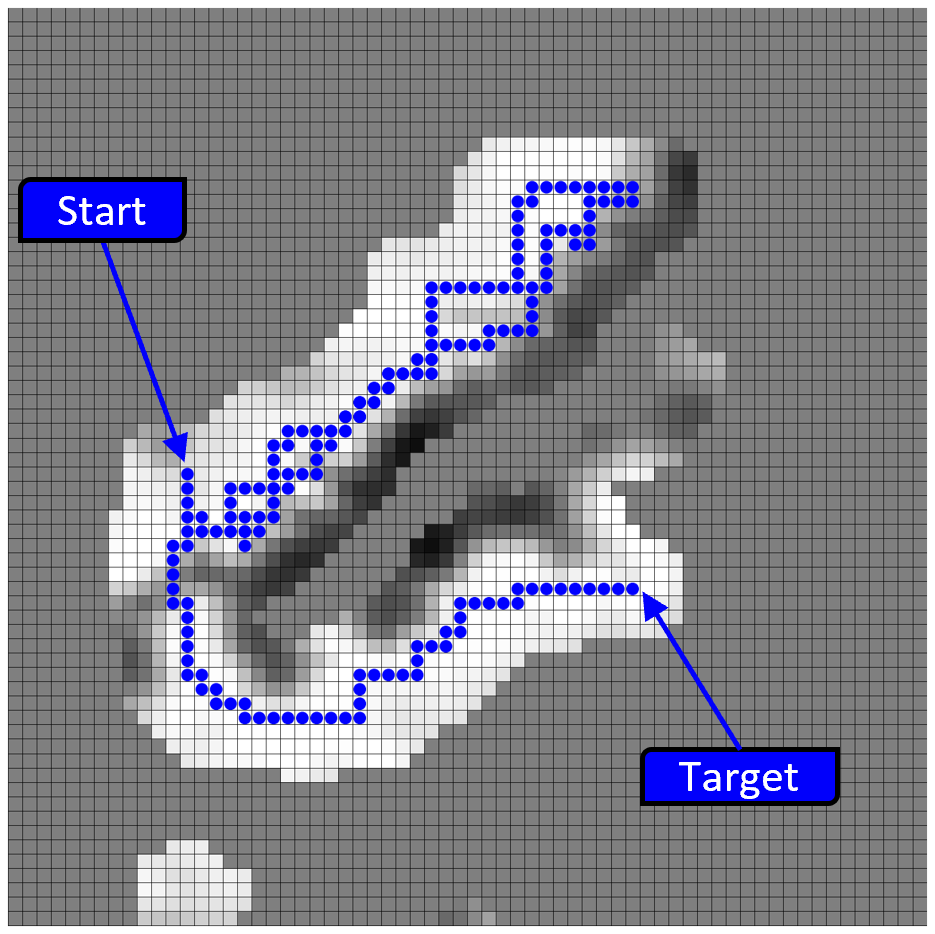}
        \\

        \raisebox{0.12\linewidth}{MI}
        &
        \includegraphics[width=0.23\linewidth]{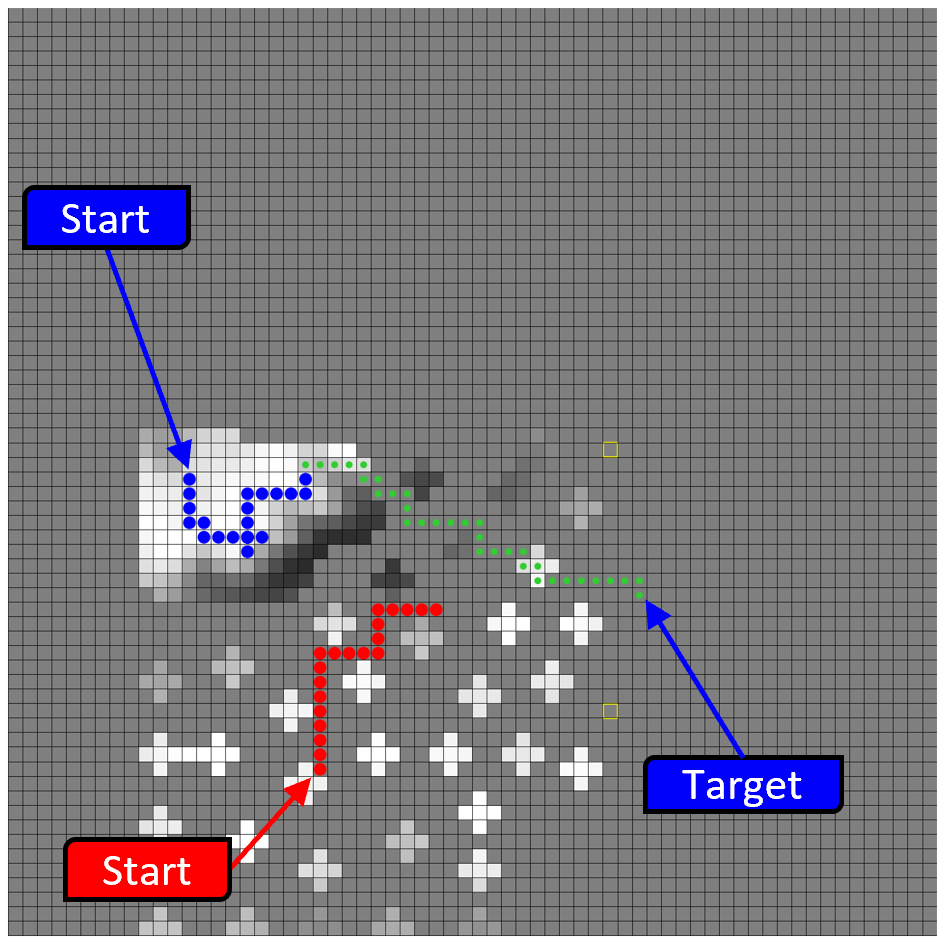}
        &
        \includegraphics[width=0.23\linewidth]{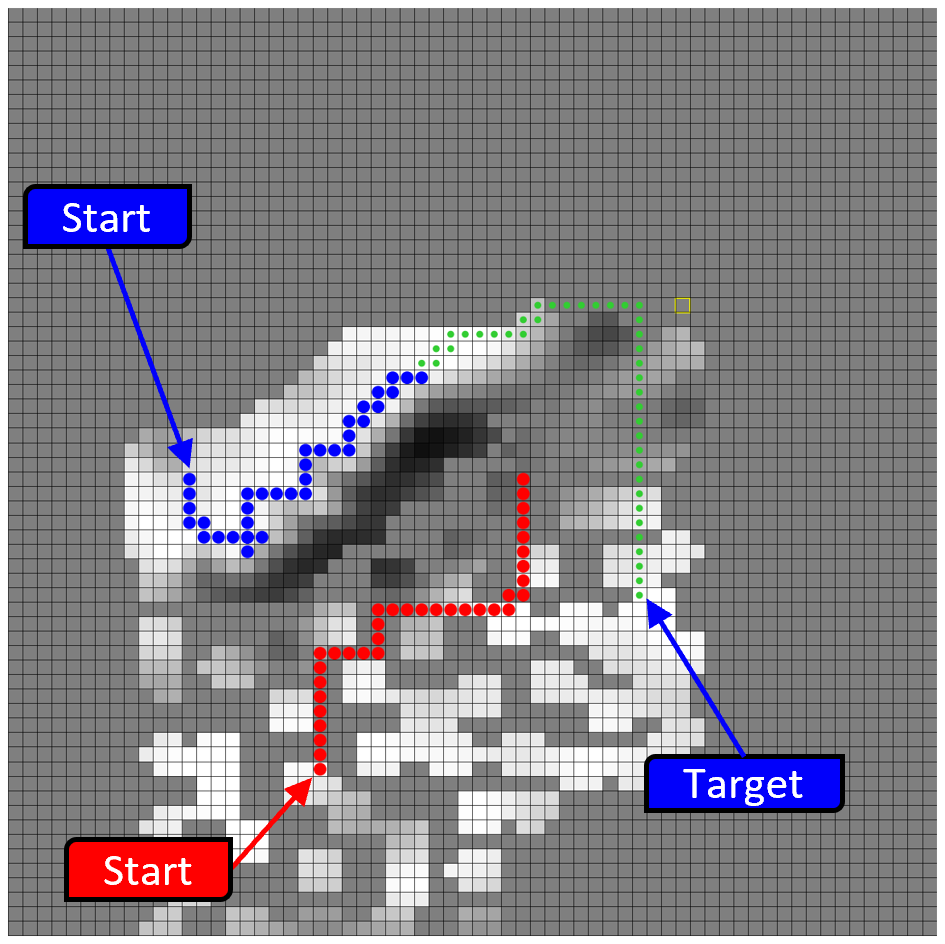}
        &
        \includegraphics[width=0.23\linewidth]{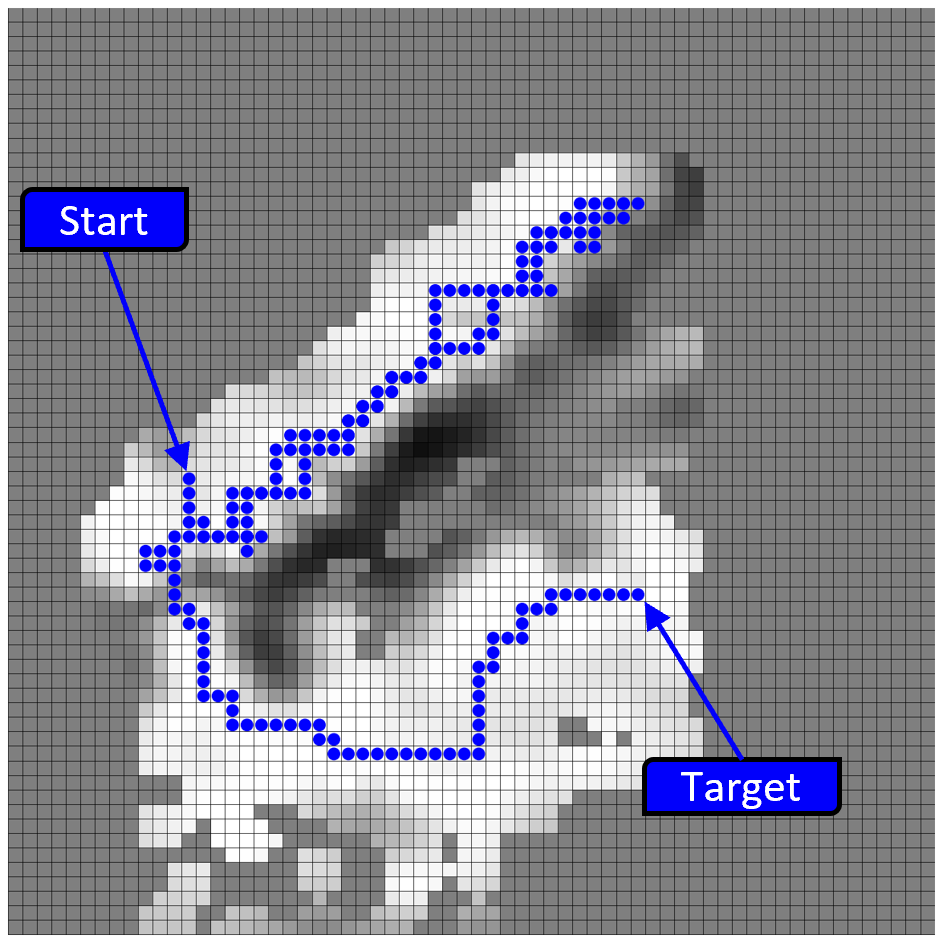}
        \\

        \raisebox{0.12\linewidth}{\shortstack{$\beta$-SGP-PA\\($\beta=1000$)}}
        &
        \includegraphics[width=0.23\linewidth]{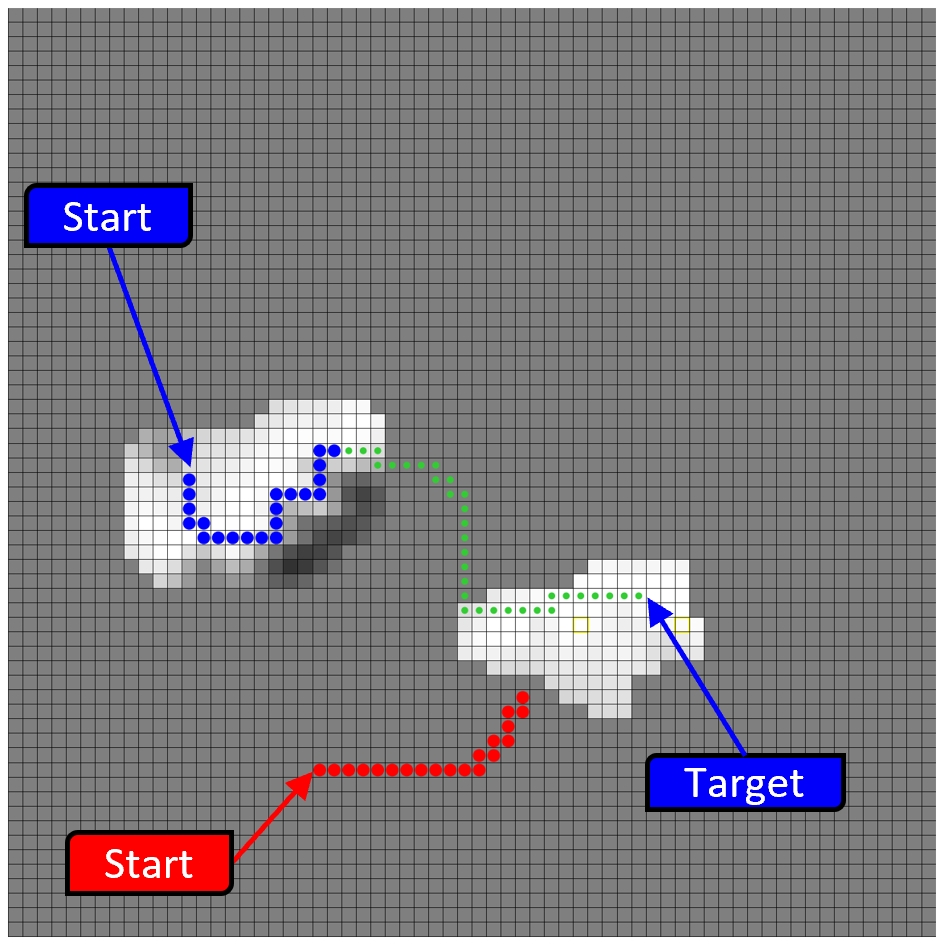}
        &
        \includegraphics[width=0.23\linewidth]{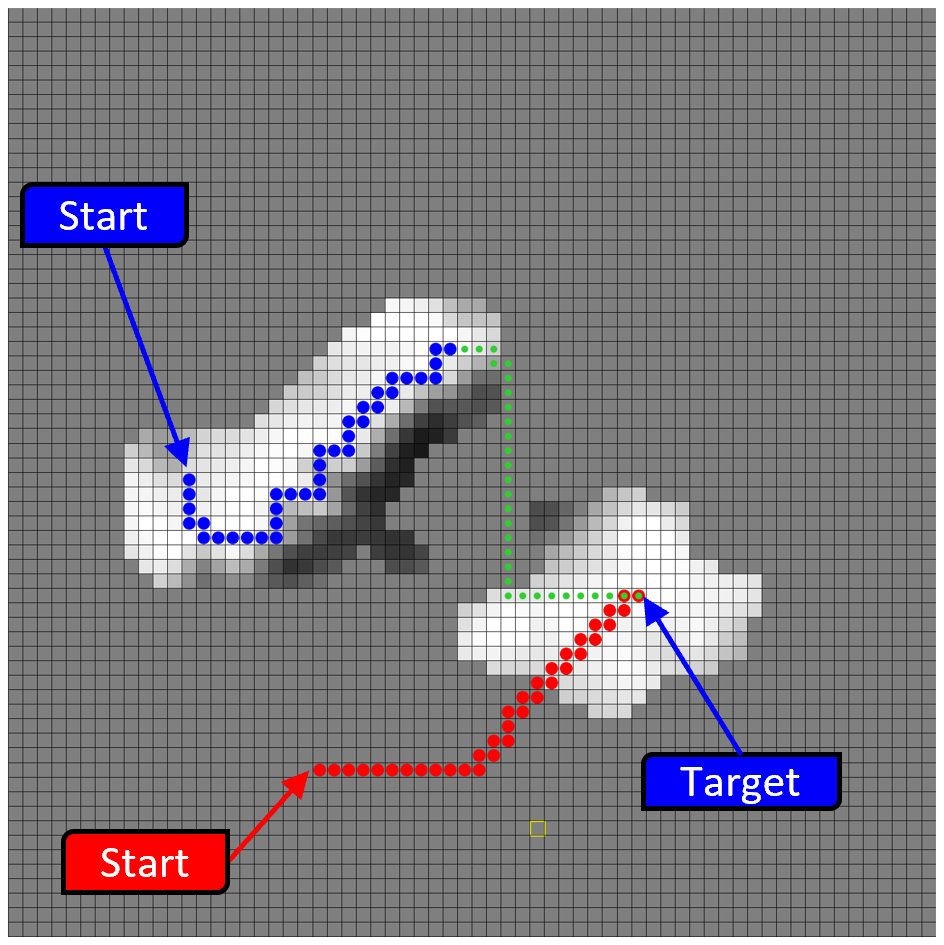}
        &
        \includegraphics[width=0.23\linewidth]{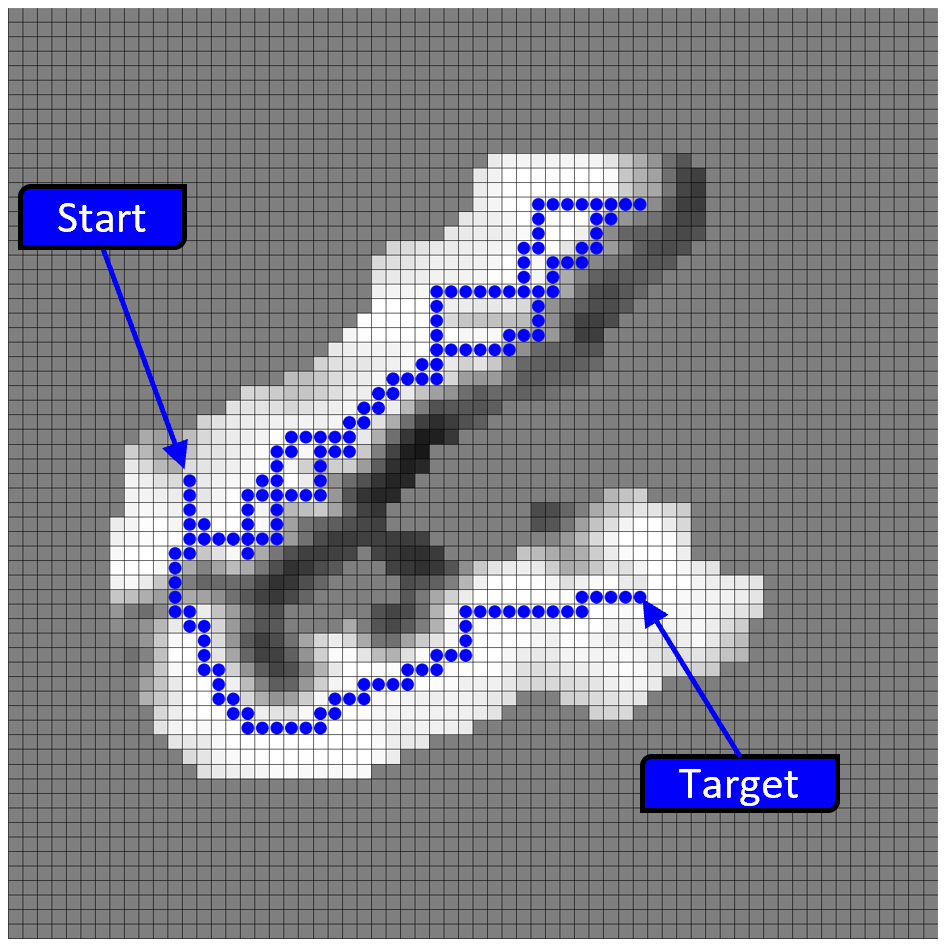}
        \\

        \raisebox{0.12\linewidth}{\shortstack{$\beta$-SGP\\($\beta=1$)}}
        &
        \includegraphics[width=0.23\linewidth]{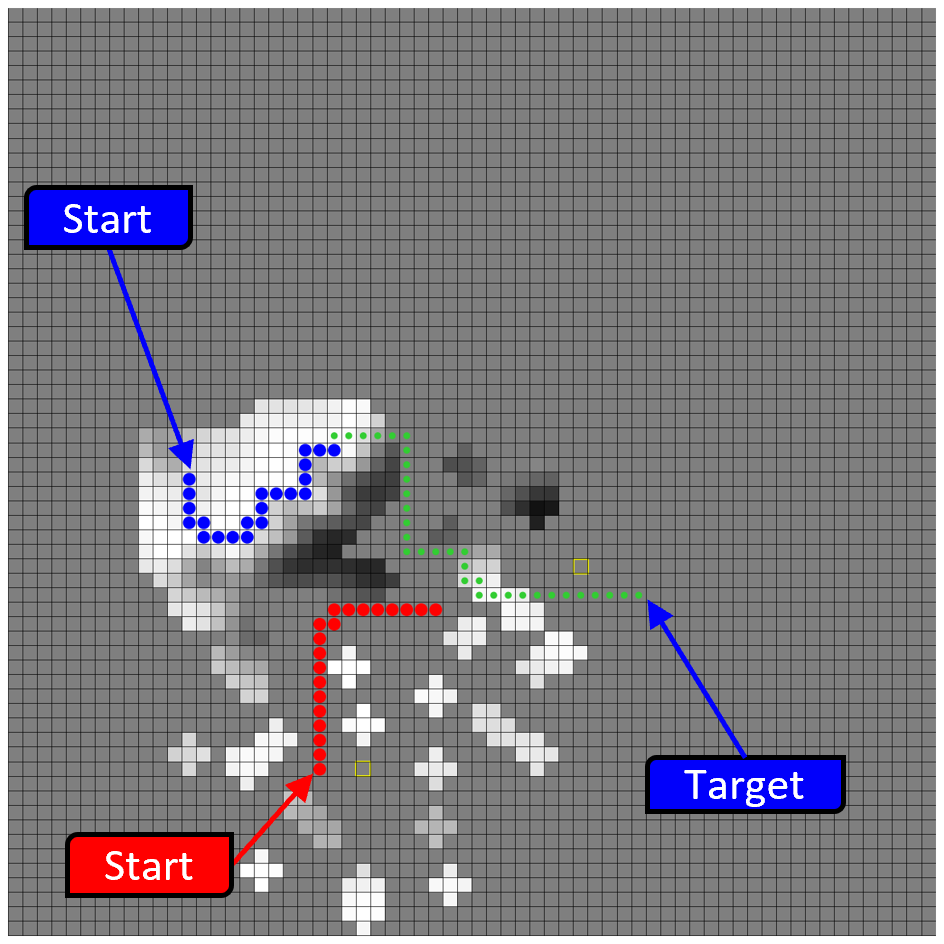}
        &
        \includegraphics[width=0.23\linewidth]{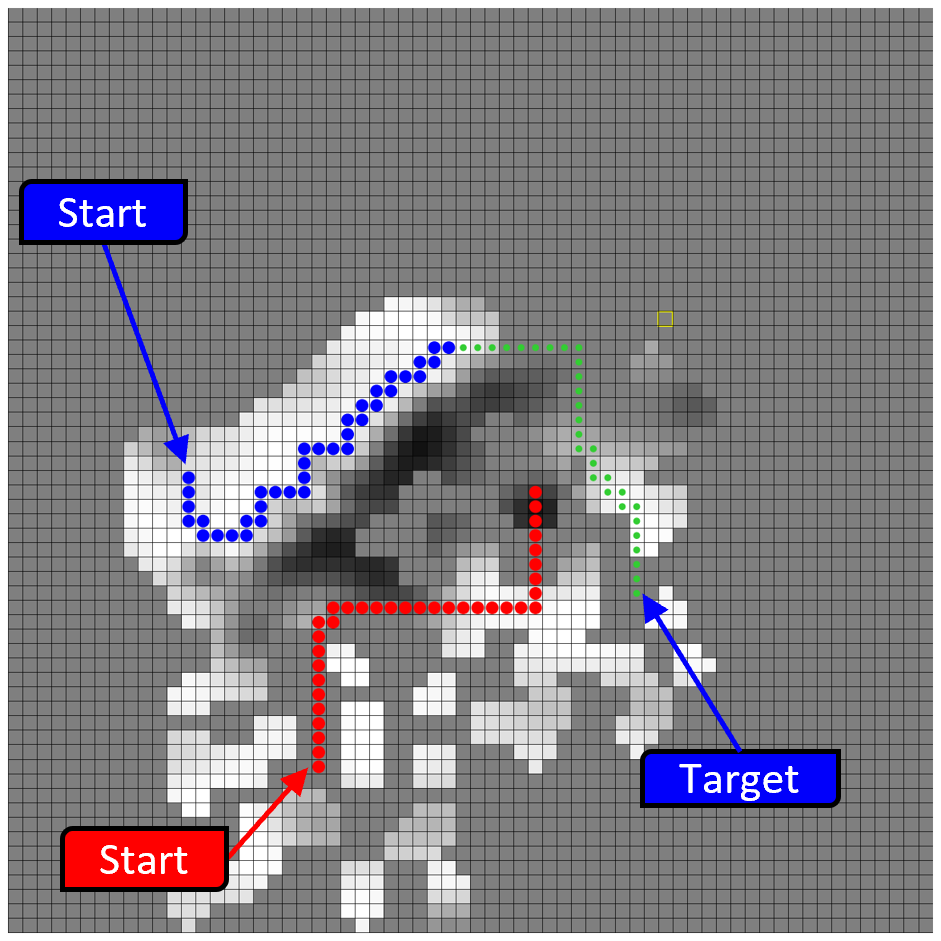}
        &
        \includegraphics[width=0.23\linewidth]{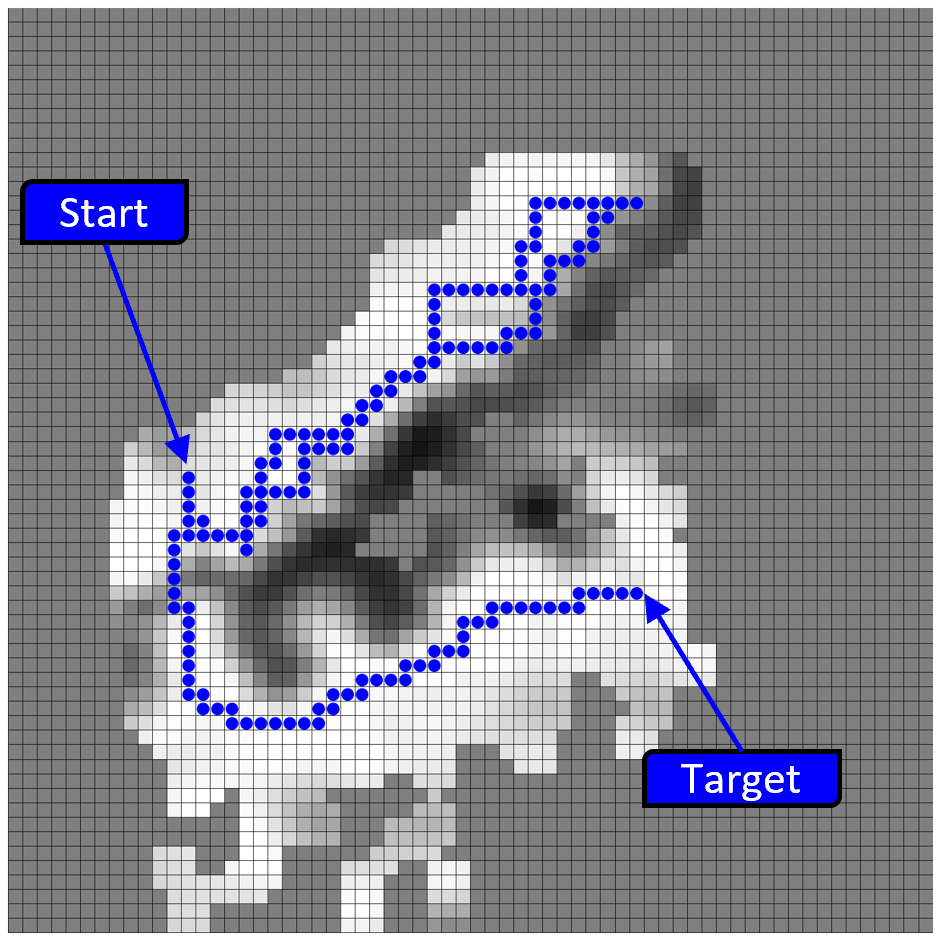}
        \\
    \end{tabular}

    \caption{(continued on the next page)}
    \label{fig:earth_sim_extended}
\end{figure}

\begin{figure}[H]
    \ContinuedFloat
    \centering
    \setlength{\tabcolsep}{2pt}
    \begin{tabular}{c c c c}
        & $t=20$ & $t=35$ & $t_\textrm{final}$ \\

        \raisebox{0.12\linewidth}{\shortstack{$\beta$-SGP\\($\beta=1000$)}}
        &
        \includegraphics[width=0.23\linewidth]{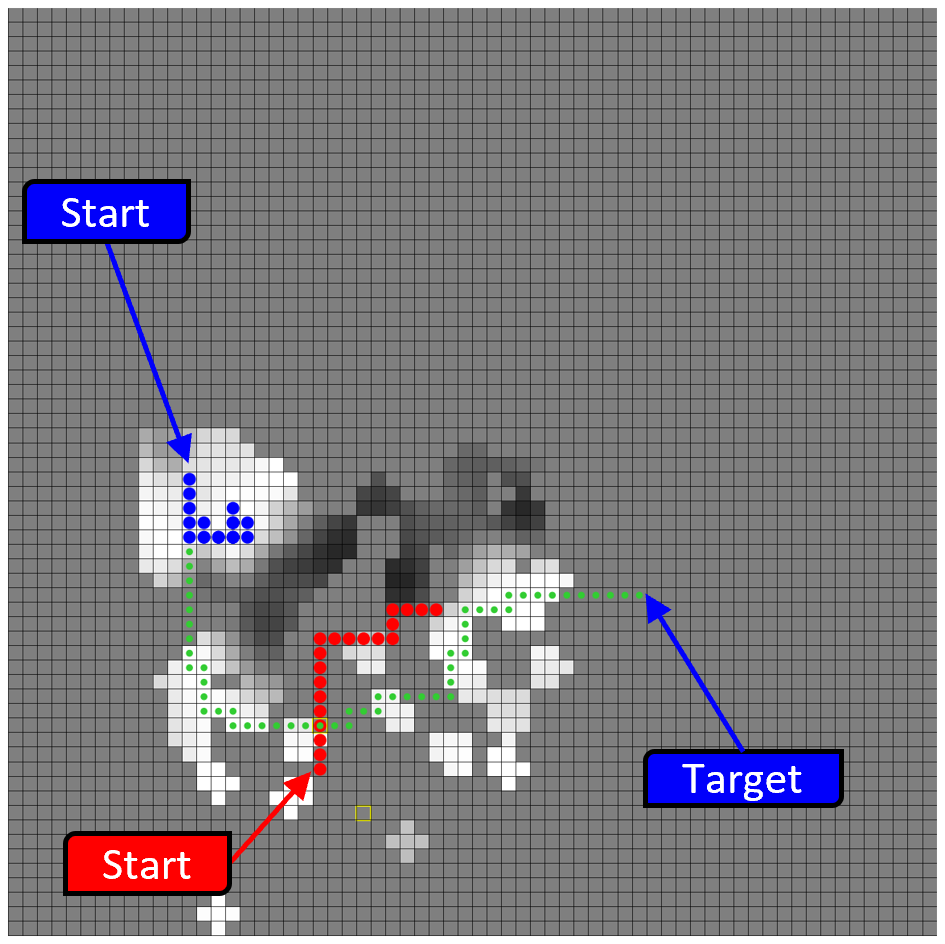}
        &
        \includegraphics[width=0.23\linewidth]{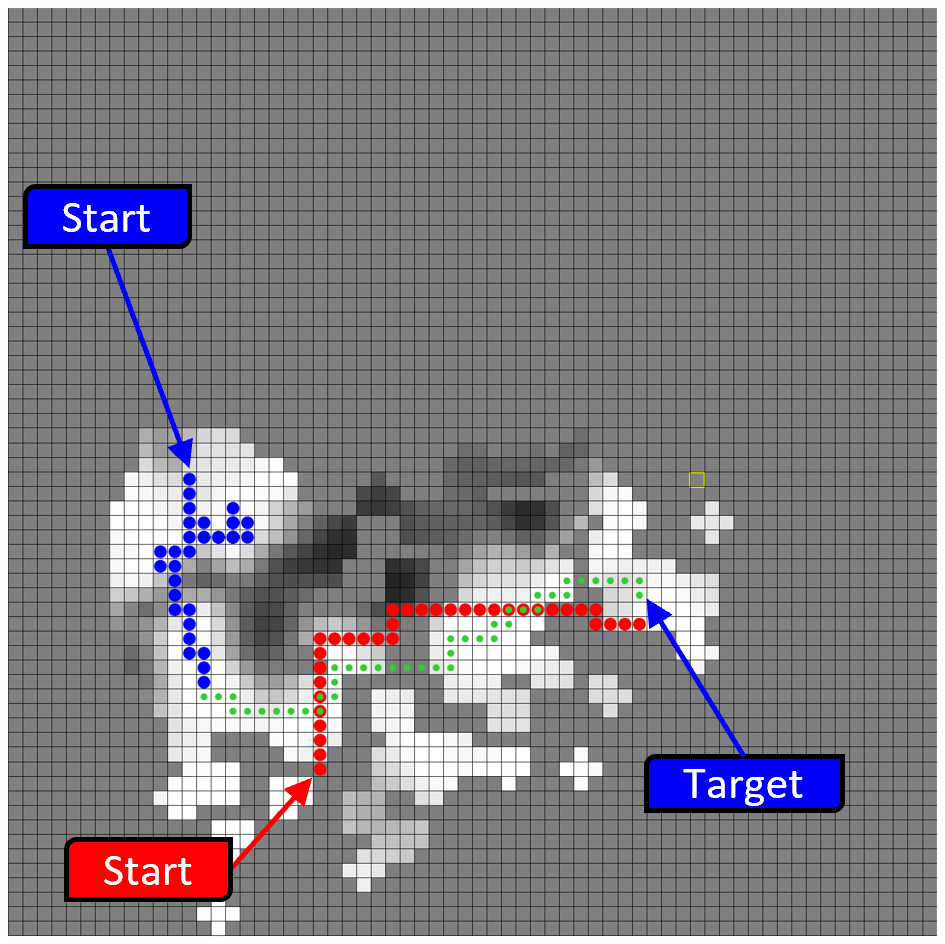}
        &
        \includegraphics[width=0.23\linewidth]{./figs/sim2/betaSGP1000_final.png}
        \\

        \raisebox{0.12\linewidth}{FI}
        &
        \includegraphics[width=0.23\linewidth]{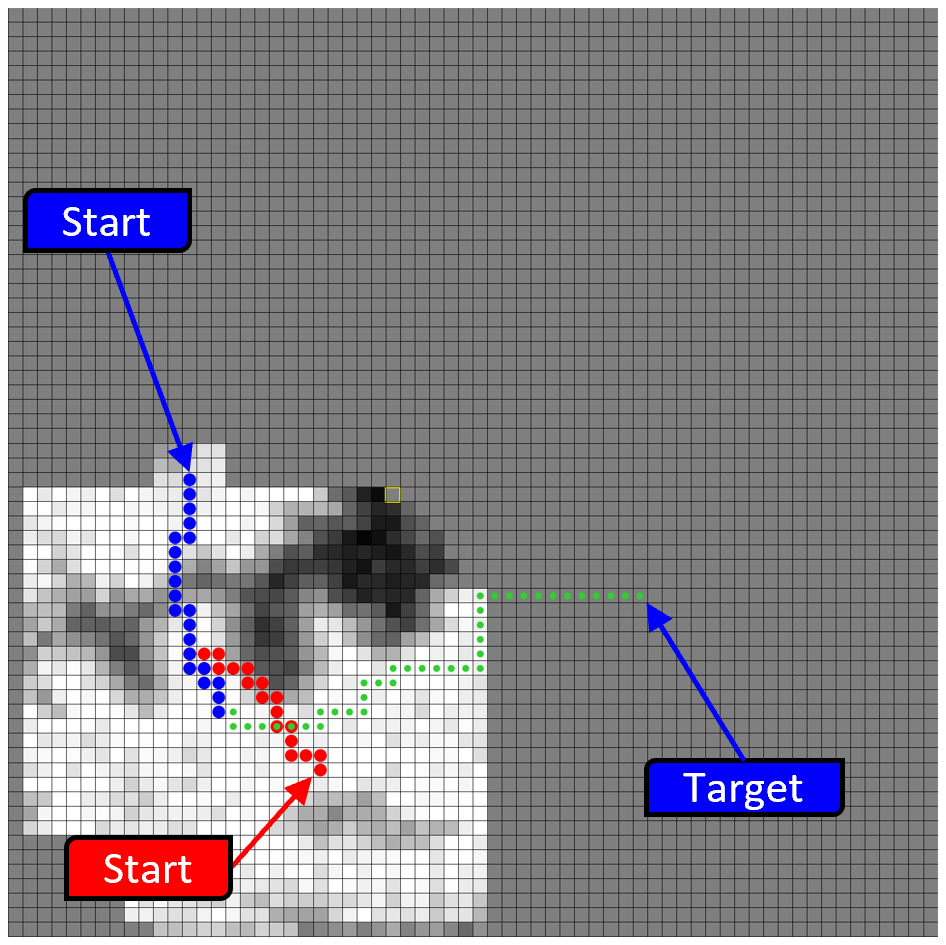}
        &
        \includegraphics[width=0.23\linewidth]{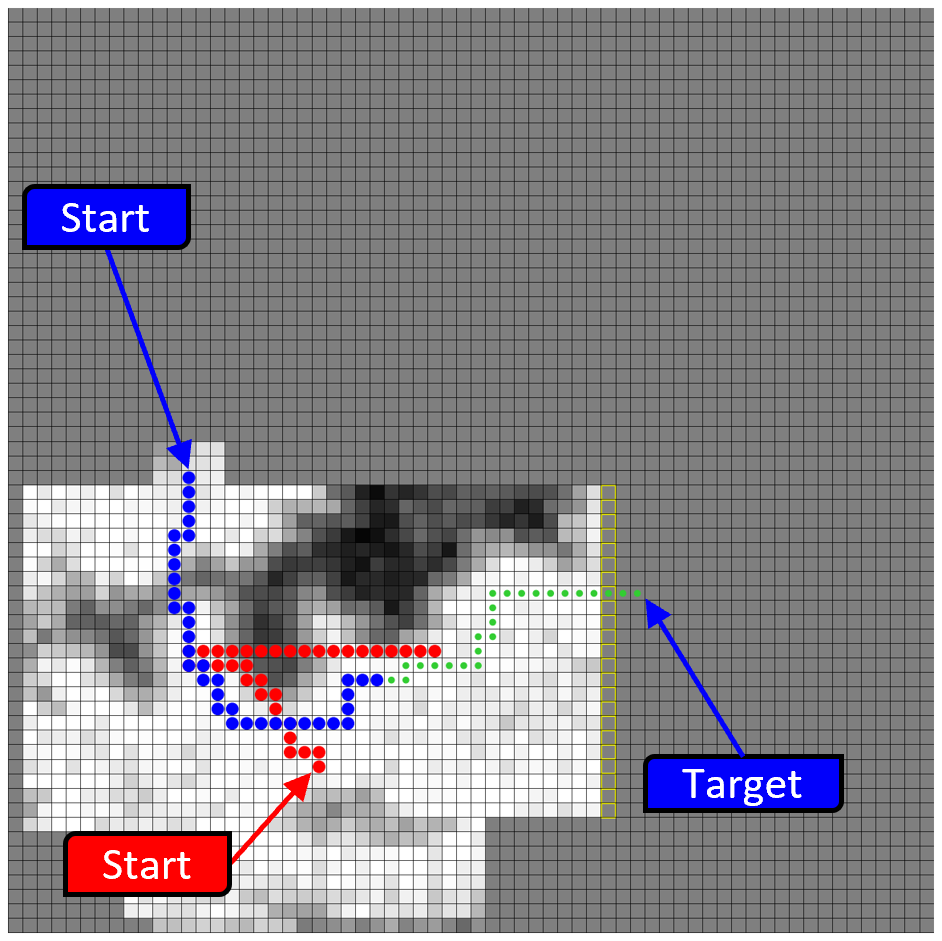}
        &
        \includegraphics[width=0.23\linewidth]{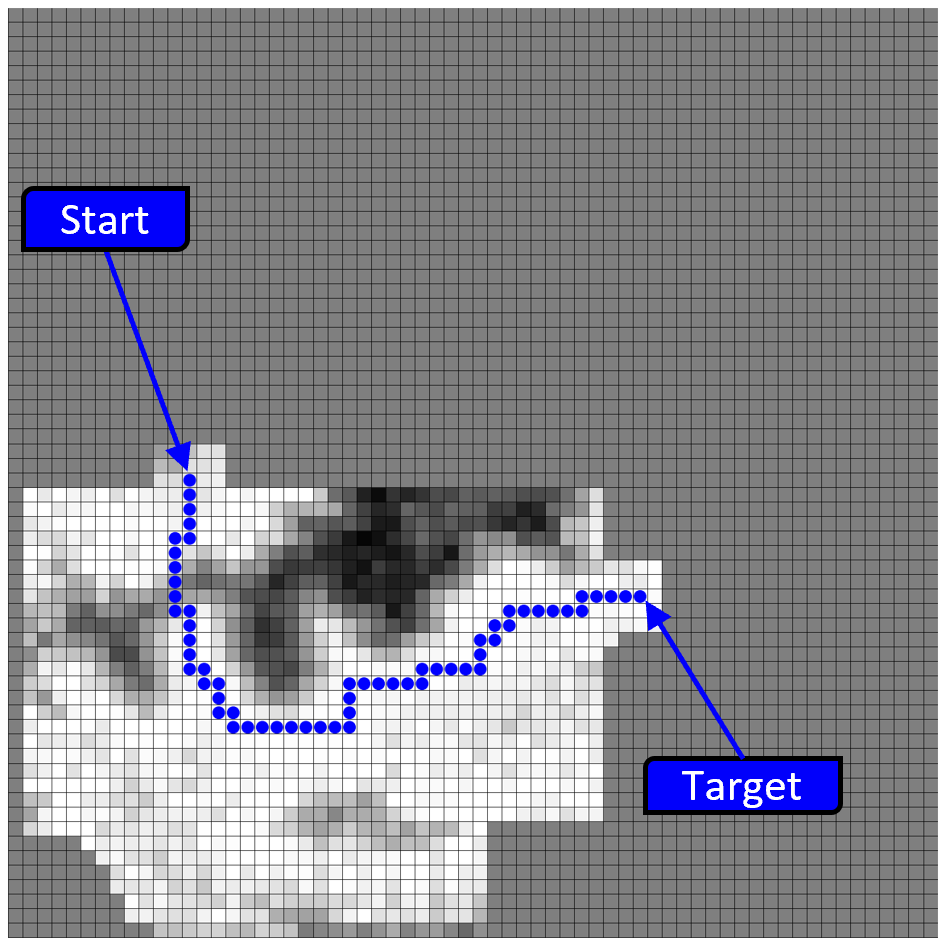}
        \\

        \raisebox{0.12\linewidth}{FI-GP}
        &
        \includegraphics[width=0.23\linewidth]{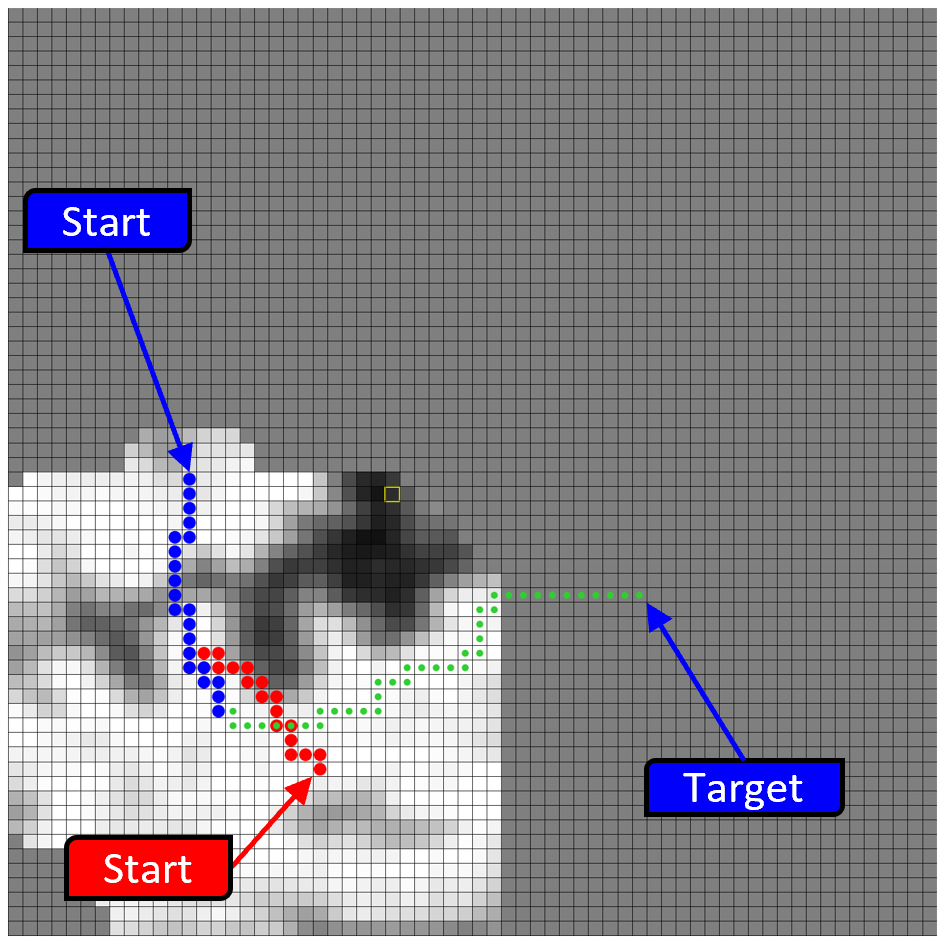}
        &
        \includegraphics[width=0.23\linewidth]{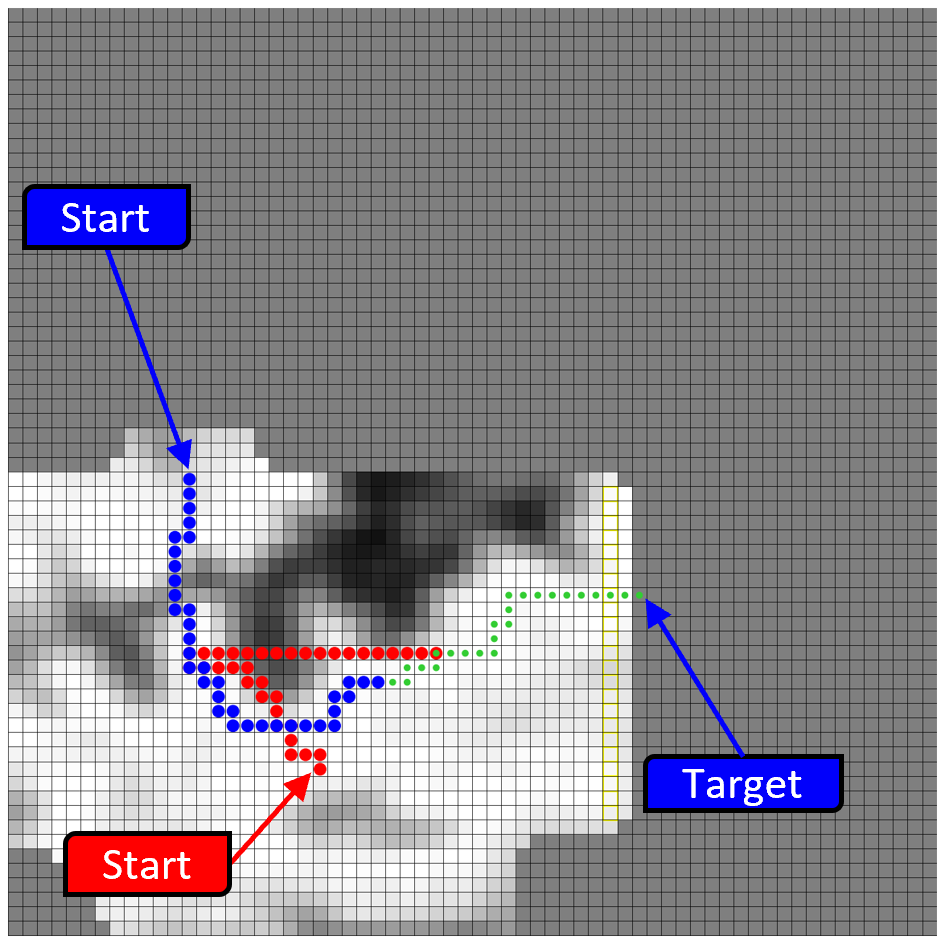}
        &
        \includegraphics[width=0.23\linewidth]{./figs/sim2/FIGP_final.png}
        \\

        \raisebox{0.12\linewidth}{FI-FO}
        &
        \includegraphics[width=0.23\linewidth]{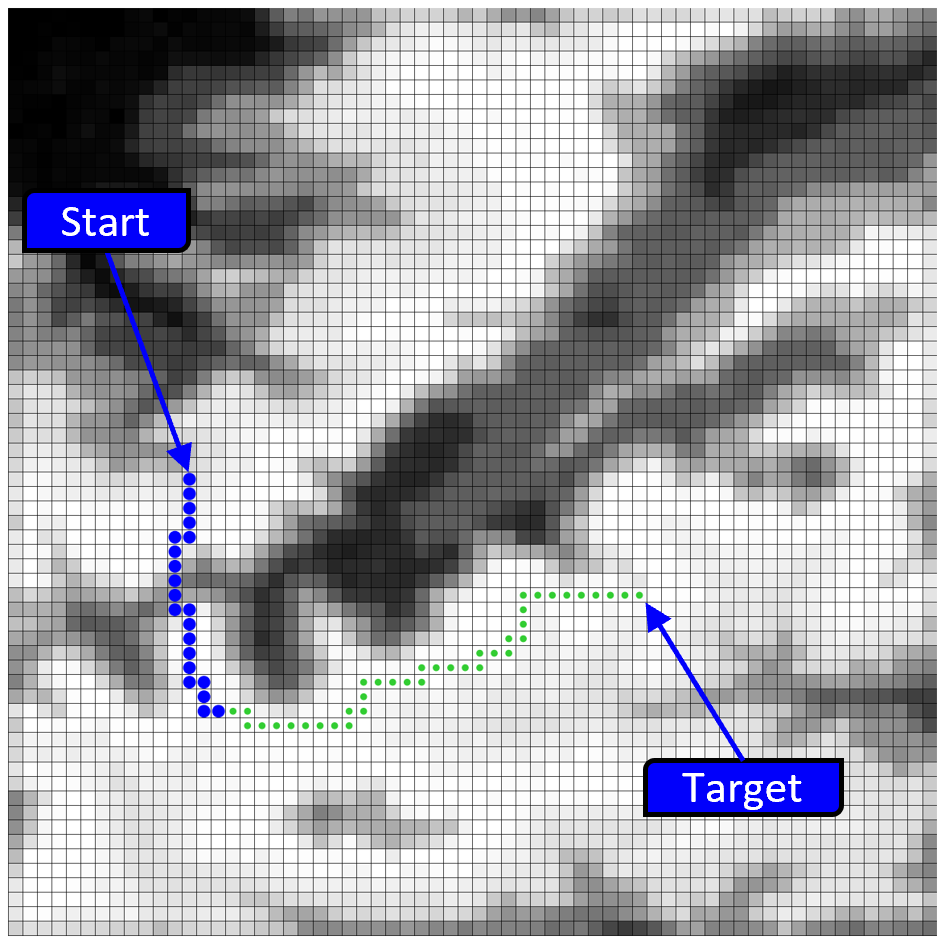}
        &
        \includegraphics[width=0.23\linewidth]{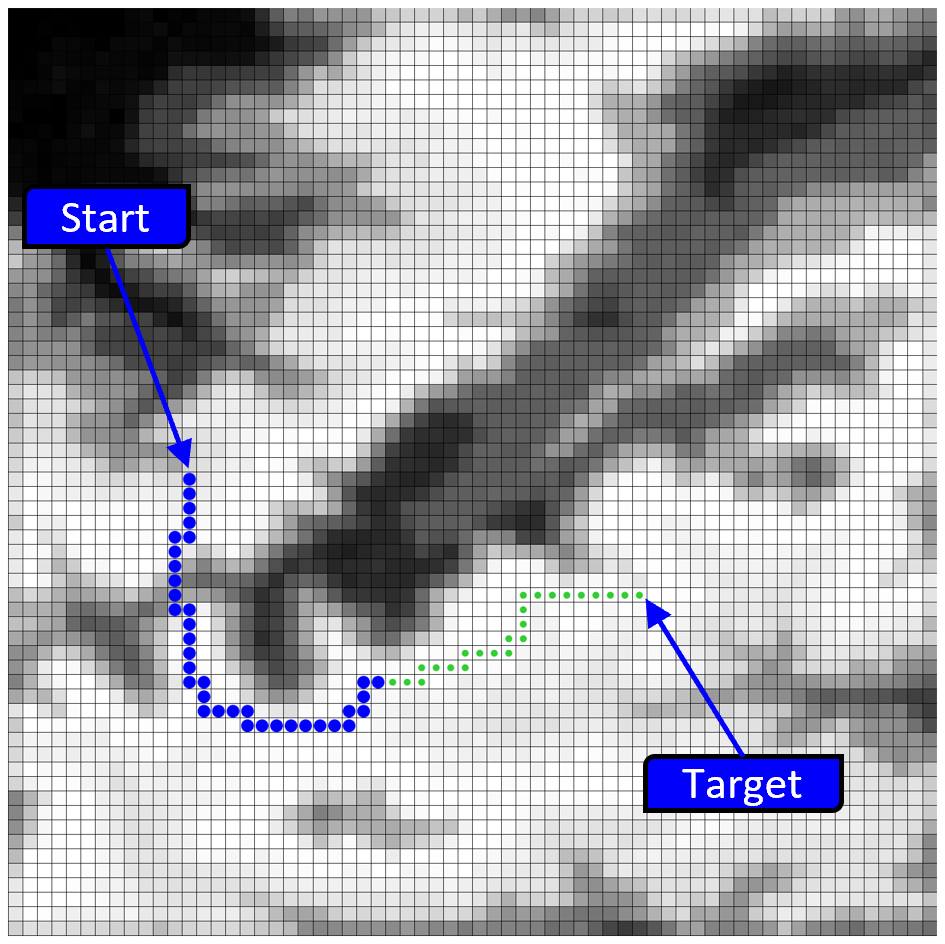}
        &
        \includegraphics[width=0.23\linewidth]{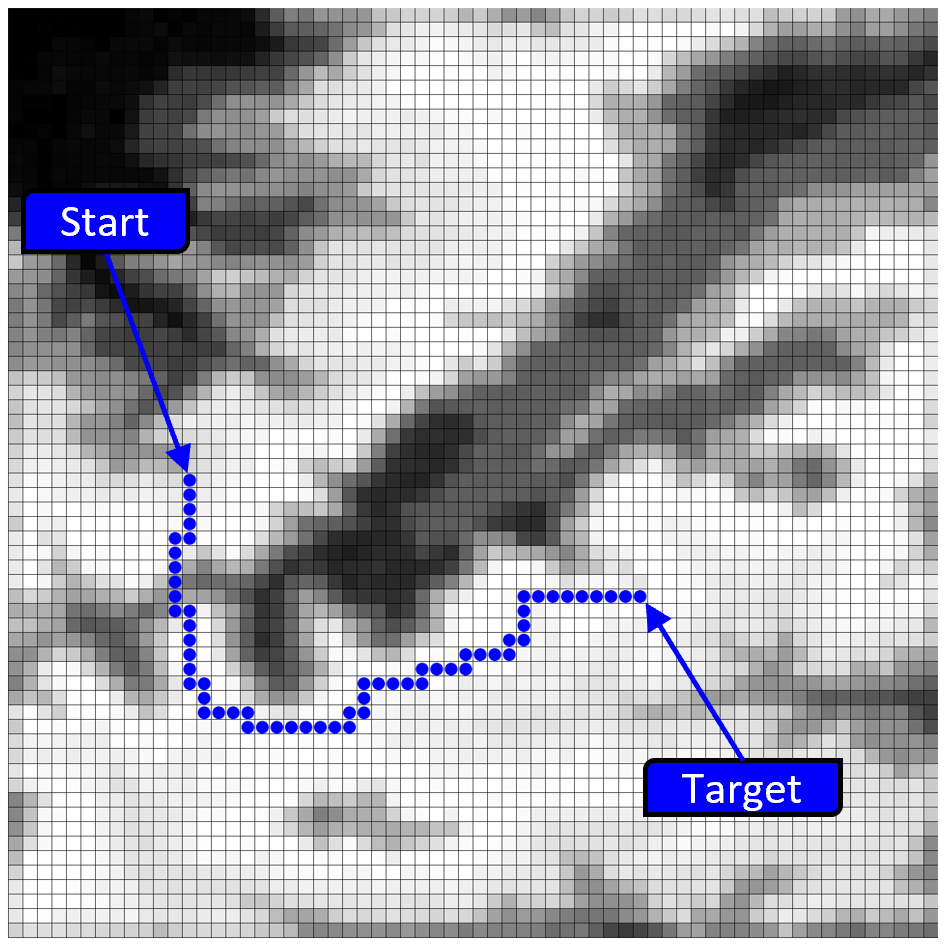}
        \\
    \end{tabular}

    \caption{Reconstructed maps (continued). Blue and red cells indicate regions traversed by the Actor and Sensor, respectively, while green cells indicate the Actor's planned path and $\mathbf{x}_0^S=(21,13)$.
    The final times for this scenario are:
    $t_{\textrm{U}}=170$,
    $t_{\textrm{Cont-SGP}}=174$,
    $t_{\textrm{MI}}=184$,
    $t_{\beta\textrm{-SGP-PA},\beta=1000}=168$,
    $t_{\beta\textrm{-SGP},\beta=1}=168$,
    $t_{\beta\textrm{-SGP},\beta=1000}=80$,
    $t_{\textrm{FI}}=58$,
    $t_{\textrm{FI-GP}}=58$, and
    $t_{\textrm{FI-FO}}=58$.}
\end{figure}

\end{document}